\documentclass{article}

\PassOptionsToPackage{numbers,sort&compress}{natbib}

\usepackage[preprint]{neurips_2026}

\usepackage[utf8]{inputenc} 
\usepackage[T1]{fontenc}    
\usepackage{hyperref}       
\usepackage{url}            
\usepackage{booktabs}       
\usepackage{amsfonts}       
\usepackage{nicefrac}       
\usepackage{microtype}      
\usepackage{xcolor}         
\usepackage{xspace}
\usepackage{amsmath}
\usepackage{amsthm}
\usepackage{mathtools}
\usepackage{amssymb}
\usepackage{enumitem}
\usepackage{wrapfig}

\newcommand{\minihead}[1]{\vspace{0.2em}\noindent\textbf{#1}}
\newcommand{\eg}{e.g.}

\newcommand{\iid}{i.i.d.\xspace}

\newtheorem{theorem}{Theorem}
\newtheorem{corollary}{Corollary}

\newtheorem{remark}{Remark}
\newcommand{\ndiffbase}{9--51\%\xspace}
\newcommand{\ndiffstudent}{6--11\%\xspace}

\newcommand{\ndiffdistill}{84--97\%\xspace}
\newcommand{\ntightopd}{15\%\xspace}

\newcommand{\sizeratio}{14,796$\times$\xspace}

\title{Non-vacuous Generalization Bounds for Reinforcement Learning with Verifiable Rewards}

\author{%
  Yuxuan Zhu$^*$ \\
  \texttt{yxx404@illinois.edu} \\
  \And
  Rohan Alur$^{\dagger\;\ddagger}$ \\
  \texttt{ralur@mit.edu} \\
  \And
  Daniel Kang$^{*\;\ddagger}$ \\
  \texttt{ddkang@illinois.edu} \\
  \AND
  $^*$UIUC \;\; 
  $^\dagger$MIT \;\;
  $^\ddagger$Bridgewater AIA Labs
}

\begin{document}

\maketitle

\begin{center}
    \vspace{-2\baselineskip}
    \href{https://github.com/uiuc-kang-lab/rlvr_generalization_bounds}{%
      \raisebox{-0.15em}{\includegraphics[height=1em]{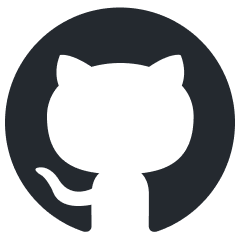}}%
      \hspace{0.45em}\texttt{Code}%
    }%
    \hspace{1.5em}%
    \href{https://huggingface.co/collections/uiuc-kang-lab/rlvr-generalization-bounds}{%
      \raisebox{-0.15em}{\includegraphics[height=1em]{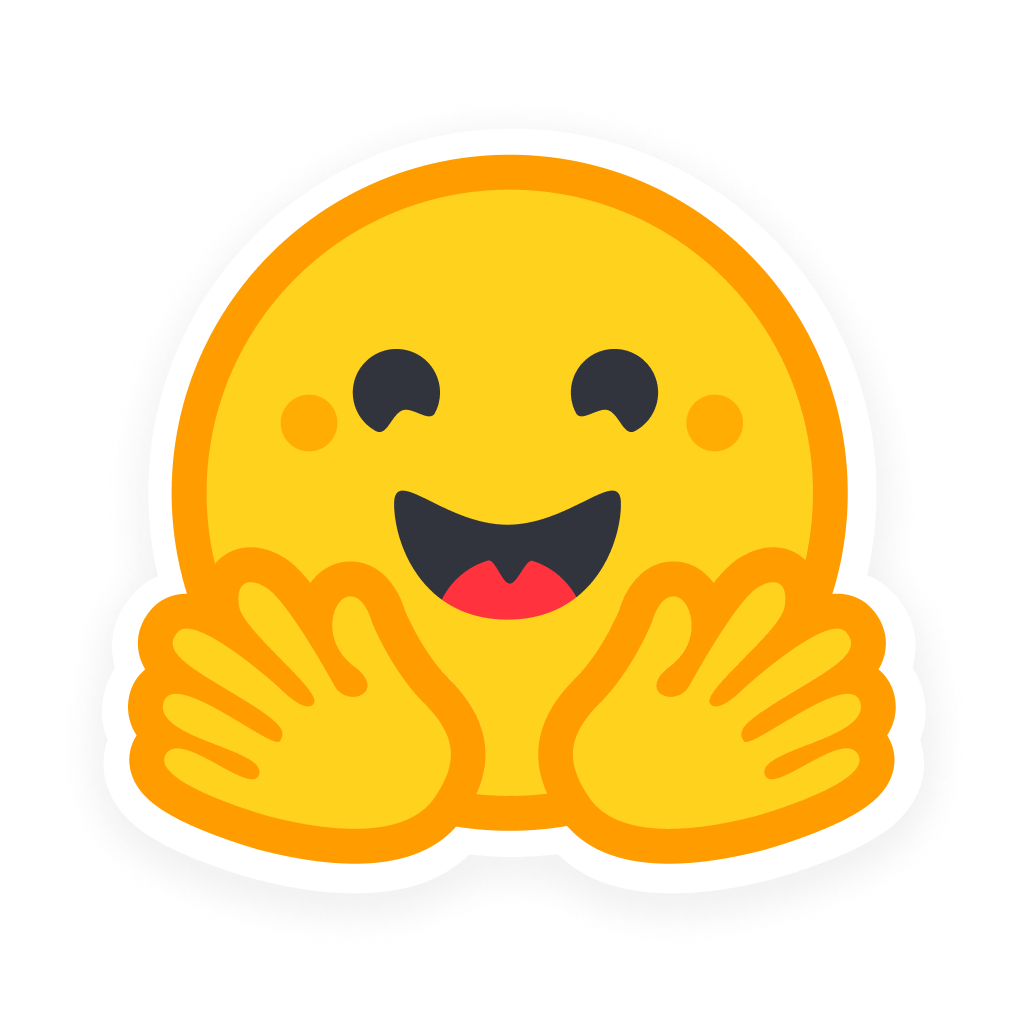}}%
      \hspace{0.45em}\texttt{Datasets \& Models}%
    }
  \end{center}

\begin{abstract}
While reinforcement learning with verifiable rewards (RLVR) is widely used to 
improve the reasoning capabilities of large language models (LLMs), the 
generalizability of the resulting models remains poorly understood. In this work, 
we establish the first non-vacuous generalization bounds for parameter-efficient 
RLVR fine-tuning at the billion-parameter scale. Our approach adapts PAC-Bayes 
compression bounds to this setting, and addresses the inherent stochasticity of 
token generation by applying the Gumbel-max reparameterization trick. To 
operationalize these bounds, we propose the Progressive RLVR framework, which 
integrates RLVR with on-policy distillation, TinyLoRA, and model quantization.
Progressive RLVR empirically retains \ndiffdistill performance of standard 
LoRA fine-tuning while producing models that are \sizeratio more compressible. 
We show that this framework yields non-vacuous generalization bounds in four 
domains: mathematical problem-solving, programming, general-knowledge reasoning, 
and Text-to-SQL. Our bounds exceed the accuracy of the base model by \ndiffbase 
and lie within \ndiffstudent of the accuracy of the fine-tuned models.
\end{abstract}

\section{Introduction}

Reinforcement learning with verifiable rewards (RLVR) has emerged as a dominant 
paradigm for fine-tuning large language models in reasoning-intensive domains
\cite{guo2025deepseek,yang2025qwen3,team2025kimi,chen2025minimax}, including 
mathematical problem-solving \cite{shao2024deepseekmath} and programming 
\cite{guo2025deepseek,wei2025swe}. However, recent work has shown that models 
fine-tuned via RLVR exhibit limited generalizability even for in-domain tasks 
\cite{hu2025breaking}, and are prone to diversity collapse and other 
pathological behaviors \cite{yue2025does,nguyen2025reasoning}. These analyses suggest that RLVR 
specializes models to their training examples rather than enabling generalizable 
reasoning capabilities \cite{hu2025breaking,nguyen2025reasoning}. Consequently, 
understanding the generalizability of RLVR models is critical for downstream 
deployment, particularly in high-stakes domains \cite{hager2024evaluation}.

Unfortunately, verifying generalizability for RLVR is
theoretically challenging. Existing techniques are either inapplicable
or yield vacuous bounds. Classic complexity measures, such as 
Vapnik-Chervonenkis (VC) dimension \cite{blumer1989learnability} and Rademacher 
complexity \cite{bartlett2002rademacher,wang2019generalization}, scale with 
parameter counts and often lead to trivial, uninformative bounds for modern LLMs. 

While prior literature has developed the compression-based Probably 
Approximately Correct Bayesian (PAC-Bayes) framework to compute non-vacuous 
generalization bounds for LLM pre-training \cite{lotfi2023non}, extending these 
guarantees to RLVR introduces several challenges. Unlike the pre-training 
setting of \citet{lotfi2023non}, where the negative log-likelihood loss is 
deterministic for each input sequence given a fixed hypothesis, RLVR rewards 
are defined over stochastically generated sequences whose support is 
combinatorial in sequence length and vocabulary size. This makes the empirical 
reward intractable and introduces a second source of randomness beyond the 
training data distribution. Together, these make the existing PAC-Bayes bounds 
\cite{rodriguez2024more,lotfi2023non,lotfi2022pac} inapplicable.

In this work, we resolve these theoretical bottlenecks and develop the first 
framework that yields non-vacuous generalization bounds for LLMs 
post-trained via parameter-efficient RLVR.

To establish generalization bounds for RLVR, we handle the stochasticity of 
the token generation process by decomposing the generalization gap into an 
extrinsic component over the input data distribution and an intrinsic component 
over the model's own generation process. To tackle the combinatorially 
large support of generated sequences, we apply a Gumbel-max reparameterization 
trick \cite{wang2019generalization,jang2016categorical,gumbel1954statistical} 
to represent token generation as a deterministic function of the input sequence, 
random noise, and hypothesis, enabling the compression-based PAC-Bayes arguments 
\cite{lotfi2022pac} on stochastic token generation.

Since the compression-based PAC-Bayes bound scales with the description length 
of the weight updates, obtaining a non-vacuous bound requires aggressive 
compression of the fine-tuned model. To address this challenge, we develop 
\textit{Progressive RLVR}, a post-training pipeline that achieves aggressive 
compression while preserving high training performance. We first train a strong 
teacher model using RLVR. We then use on-policy distillation 
\cite{agarwal2024policy,lu2025onpolicydistillation} to compress the capability 
of the teacher model into a student model with fewer trainable parameters. 
During on-policy distillation, we apply TinyLoRA \cite{morris2026learning}, 
which reparameterizes the weight updates into as few as one trainable scalar during 
training. Empirically, Progressive RLVR achieves 10\% higher training rewards 
(Figure~\ref{fig:distill-vs-rlvr}) and produces \ntightopd tighter bounds than 
directly applying TinyLoRA to RLVR (Section~\ref{sec:non-vacuous}), due to 
the dense token-level supervision that on-policy distillation provides.

In our experiments, we achieve \emph{the first non-vacuous generalization bounds
for RLVR} on four real-world domains, including mathematical
problem-solving, programming, general-knowledge reasoning, and Text-to-SQL 
translation. Across four domains, our generalization bounds exceed the accuracy 
of the base model by \ndiffbase, and are within \ndiffstudent of the empirical 
accuracy of the fine-tuned models. Furthermore, Progressive RLVR achieves 
excellent performance with high compression rates, recovering \ndiffdistill of 
the teacher models' accuracy, while being \sizeratio more compressible than the 
corresponding teacher model on average.

We summarize our contributions as follows:
\vspace{-0.5em}
\begin{enumerate}[leftmargin=*]
    \item \textbf{Generalization bounds for RLVR}. We establish the first 
    compression-based PAC-Bayesian framework for bounding generalization errors 
    of RLVR, addressing the stochastic reward structure of RLVR via Gumbel-max 
    reparameterization. 
    \item \textbf{Progressive RLVR}. We introduce Progressive RLVR, 
    a parameter-efficient post-training pipeline that leverages TinyLoRA and 
    on-policy distillation to achieve the parameter compression required for 
    tight bounds without significantly sacrificing accuracy.
    \item \textbf{Non-vacuous bounds for billion-parameter models on real-world 
    RLVR datasets}. We provide the first non-vacuous generalization bounds for
    billion-parameter LLMs post-trained via RLVR, providing formal verification
    of RLVR performance on real-world tasks including mathematical 
    problem-solving, programming, general-knowledge reasoning, and Text-to-SQL translation.
\end{enumerate}

\section{Preliminaries}
\label{sec:preliminaries}

In this section, we present the foundational concepts and formulations that 
motivate our framework, specifically focusing on reinforcement learning with 
verifiable rewards (Section \ref{subsec:prelim-rlvr}), the PAC-Bayes approach 
to generalization bounds (Section \ref{subsec:prelim-pacbayes}), and LLM 
fine-tuning with LoRA (Section \ref{subsec:prelim-peft}).

\subsection{Reinforcement Learning with Verifiable Rewards}
\label{subsec:prelim-rlvr}

RLVR has emerged as a dominant paradigm of LLM post-training to enhance 
reasoning capabilities \cite{guo2025deepseek,chen2025minimax,team2025kimi}. 
Relying on objective verifiers (e.g., mathematical correctness 
\cite{shao2024deepseekmath} or unit tests \cite{wei2025swe}) rather than human 
preference models \cite{bai2022training}, RLVR provides deterministic 
reward signals. While the practical instantiations of RLVR 
\cite{shao2024deepseekmath,yu2025dapo,gao2025soft,chen2025minimax} rely on 
various customized clipping mechanisms and optional KL-divergence penalties to 
stabilize training, all RLVR variants share the core optimization objective of 
classic reinforcement learning: optimizing a policy to maximize the expected reward. 

In the context of outcome-based RLVR \cite{shao2024deepseekmath}, the return 
simplifies to a deterministic scalar reward evaluated on the final generated 
response. We write the objective as
\begin{equation*}
    \mathcal{J}_{\text{simplified}}(\theta) = \mathbb{E}_{x \sim \mathcal{P}_0, y \sim \pi_\theta(\cdot \mid x)} \left[ r(x, y) \right],
\end{equation*}
where $x$ represents the prompt drawn from training distribution $\mathcal{P}_0$, 
and $y$ represents the generation sampled from the model's conditional policy 
$\pi_\theta$. We use this objective as the primary expected risk measure to 
establish our generalization bounds in Section \ref{sec:bound}.

\subsection{PAC-Bayesian Generalization and Compression}
\label{subsec:prelim-pacbayes}

The PAC-Bayes framework \cite{alquier2024user} has emerged as 
a critical tool for establishing generalization bounds for deep neural 
networks, overcoming the limitations of parameter-scaling measures 
\cite{blumer1989learnability}. It bounds the generalization gap via the 
divergence between a prior and a posterior distribution defined over a 
hypothesis space $\mathcal{H}$. A central feature of this framework is that the 
bound penalizes the chosen hypothesis only via its likelihood under a pre-defined 
prior and its empirical performance. The procedure used to select the hypothesis 
from $\mathcal{H}$ is unconstrained.

To obtain tight bounds, we follow prior work \cite{lotfi2022pac,lotfi2023non} to 
apply the universal Solomonoff prior \cite{solomonoff1964formal}, where the 
prior probability of a hypothesis scales inversely with its description length 
(compressibility). In the post-training setting, the parameter update (the delta 
in parameter space) contains all the information needed to reconstruct the 
trained model from the pre-trained checkpoint. Therefore, obtaining a tight 
bound for massive LLMs requires reducing the parameter updates to a minimal 
number of parameters. A second challenge specific to RLVR is the intractability 
of the reward structure. In RLVR, the reward is calculated over the 
stochastically generated tokens $y \sim \pi_\theta(\cdot \mid x)$, whose support 
is combinatorial in length and vocabulary size. We address this challenge in 
Section \ref{sec:bound}.

\subsection{Low-Rank Adaptation (LoRA)}
\label{subsec:prelim-peft}

LoRA \cite{hu2022lora} is one of the most popular parameter-efficient 
fine-tuning (PEFT) techniques \cite{han2024parameter} for LLMs 
\cite{han2024parameter,wang2025parameter,xu2026parameter}. Given a frozen 
pre-trained weight $W_0 \in \mathbb{R}^{d \times k}$, LoRA introduces two 
trainable matrices $A \in \mathbb{R}^{r \times k}$ and 
$B \in \mathbb{R}^{d \times r}$ (where rank $r \ll \min(d, k)$). During training,
LoRA modifies the forward pass as: $W_0 x + \Delta W x = W_0 x + \frac{\alpha}{r} B A x$.

While standard LoRA reduces the parameter count, obtaining non-vacuous PAC-Bayes 
compression bounds requires even more extreme compression. To achieve this, we 
adopt the TinyLoRA parameterization \cite{morris2026learning}. 
Building on the SVD-based parameterization of LoRA-XS \cite{balazy2024lora}, 
TinyLoRA leverages the truncated singular value decomposition of the frozen 
pre-trained weight $W_0 \approx U \Sigma V^T$. It replaces the trainable core 
matrix with a low-dimensional trainable vector $v \in \mathbb{R}^u$ projected 
through a set of fixed random matrices $P_i \in \mathbb{R}^{r \times r}$. The 
weight update $\Delta W$ is then reparameterized as:
\begin{equation*}
    \Delta W = U \Sigma \left( \sum_{i=1}^u v_i P_i \right) V^T.
\end{equation*}
With TinyLoRA, we can compress the learned policy update to as few as one 
trainable scalar \cite{morris2026learning}. This extreme parameterization forms 
the architectural basis of our Progressive RLVR framework (Section 
\ref{sec:progressive-RLVR}), allowing us to compute tight compression-based 
bounds.

\section{Generalization Bounds for RLVR}
\label{sec:bound}

We now develop generalization bounds for LLMs post-trained via RLVR. We begin 
by introducing notation and setup (Section \ref{sec:rlvr-bounds}). Then, we 
present our main result (Theorem~\ref{theorem:bound}), which bounds the 
generalization gap by decomposing it into an extrinsic component over the prompt 
distribution and an intrinsic component over the model's generation process.
Next, we discuss two extensions of Theorem~\ref{theorem:bound}: multi-turn
RLVR and out-of-domain (OOD) evaluation (Section~\ref{sec:multi-turn}).
Finally, we show that the empirical reward can be efficiently estimated from a 
subsample of training prompts with a mild additive penalty 
(Section \ref{sec:subsample-bounds}).

\subsection{PAC-Bayes Compression Bounds for RLVR}
\label{sec:rlvr-bounds}

\minihead{Notation and setup.}
Let $\mathcal{P}_0$ denote a distribution over task prompts, and let
$\hat{\mathcal{P}}_0$ denote the uniform distribution over a fixed set of $m$ 
training prompts $\{x_1, \dots, x_m\}$ drawn \iid from $\mathcal{P}_0$. 
Let $\pi_{\theta_{\mathrm{ref}}}$ be a fixed pre-trained base model 
parameterized by $\theta_{\mathrm{ref}}$, and let $\mathcal{H}$ be a finite 
hypothesis space in which each hypothesis $h \in \mathcal{H}$ specifies a 
parameter update yielding a fine-tuned model $\pi_{\theta(h)}$.
Following \citet{lotfi2023non}, we adopt the Solomonoff prior 
\cite{solomonoff1964formal}, defined as $P(h) = Z^{-1} \cdot 2^{-K(h)}$, where $K(h)$ 
denotes prefix Kolmogorov complexity \cite{kolmogorov1965three} and $Z \le 1$ is 
a normalizing constant \cite{lotfi2022pac}.

We define the \emph{population reward} of a hypothesis $h$ as the expected value 
of a deterministic, bounded reward function 
$r \colon \mathcal{X} \times \mathcal{Y} \to [a,\, a + \Delta]$ where 
$\mathcal{X}$ is the prompt space and $\mathcal{Y}$ is the generation space:
\begin{equation*}
    \mathcal{R}(h)
    \;=\;
    \mathbb{E}_{x \sim \mathcal{P}_0}
    \!\left[
        \mathbb{E}_{y \sim \pi_{\theta(h)}(\cdot \mid x)}
        \!\left[ r(x, y) \right]
    \right].
\end{equation*}
Since PAC-Bayes bounds depend on $h$ only through its prior probability $P(h)$ and empirical 
reward (not on the procedure that produced $h$), the reward used in the bound 
need not match the training objective. This enables us to establish bounds on 
task accuracy throughout, even when training optimizes a richer objective. We 
leverage this flexibility in Section \ref{sec:results}.

\minihead{Reparameterization of stochastic decoding.}
While we can bound the outer expectation (over $\mathcal{P}_0$) in 
$\mathcal{R}(h)$ using standard generalization bounds, the inner expectation is 
hypothesis-dependent and taken over sequences 
$y \sim \pi_{\theta(h)}(\cdot \mid x)$, a distribution whose support is 
combinatorial in the sequence length and vocabulary size. This
makes direct evaluation or analytic bounding intractable. To obtain a tractable 
form, we reparameterize the token decoding process by externalizing the 
stochasticity into a fixed, hypothesis-independent noise distribution.

Concretely, consider the autoregressive generation of a sequence 
$y_i = (y_{i,1}, \dots, y_{i, T})$ with a bounded sequence length $T$, a 
temperature $\tau > 0$ and a vocabulary size $V$. At each step $t$, the model 
produces logits $\ell_{i, t} \in \mathbb{R}^V$ and samples a token from the 
categorical distribution
\begin{equation*}
    y_{i,t} \sim \mathrm{Categorical}\!\left(
    \mathrm{softmax}(\ell_{i, t} / \tau)\right),
\end{equation*}
whose probabilities depend on the hypothesis $h$ through the logits. The 
Gumbel-max theorem \citep{gumbel1954statistical,jang2016categorical} provides 
an equivalent construction that externalizes this stochasticity into a fixed 
noise distribution. Namely, for $\xi_{i, t}^{(v)} 
\overset{\mathrm{iid}}{\sim} \mathrm{Gumbel}(0, 1)$,
\begin{equation*}
    y_{i,t} \;=\; \arg\max_{v \in [V]}\!\left(
    \tau^{-1}\ell_{i, t}^{(v)} + \xi_{i, t}^{(v)}\right)
\end{equation*}
is distributed identically to the categorical sample above. 

Collecting all noise variables across tokens and generations, we denote
$\xi_{i,j} = (\xi_{i,j,1}, \dots, \xi_{i,j,T}) \in \mathbb{R}^{T \times V}$
with $\xi_{i,j,t}^{(v)} \overset{\mathrm{iid}}{\sim} \mathrm{Gumbel}(0,1)$ 
and $j \in [n]$ indexing the $n$ independent generation samples per prompt. 
The token generation process becomes a deterministic function 
$y_{i,j} = g(x_i, \xi_{i,j}; h)$ of the prompt, noise, and hypothesis. We 
denote the product distribution over each $\xi_{i,j}$ by $\mathcal{P}_\xi$. 
Under this reparameterization, the population reward can be equivalently 
written as
\begin{equation*}
    \mathcal{R}(h)
    \;=\;
    \mathbb{E}_{x \sim \mathcal{P}_0}
    \left[
        \mathbb{E}_{\xi \sim \mathcal{P}_{\xi}}
        \left[ r(x,\, g(x, \xi;\, h)) \right]
    \right].
\end{equation*}

With this reparameterization, both sources of randomness in $\mathcal{R}(h)$, 
the prompt distribution $\mathcal{P}_0$ and the generation noise 
$\mathcal{P}_\xi$, are now explicit and hypothesis-independent. The following 
theorem exploits this structure to bound the gap between the population reward 
and its empirical estimate. We defer the proof of Theorem \ref{theorem:bound}
to Appendix \ref{subsec:app-bound-proof}.

\begin{theorem}[PAC-Bayes bound for RLVR]
\label{theorem:bound}

For any $\delta \in (0, 1)$, we have 
\begin{align*}
    \mathbb{P}\bigg[\forall h \in \mathcal{H}, \mathcal{R}(h) \ge 
    & \frac{1}{mn} \sum_{i = 1}^m \sum_{j = 1}^n r\left(x_i, g\left(x_i, \xi_{i,j}; h\right)\right) \\ 
    & - \Delta \left(\sqrt{\frac{1}{n}} + 1\right) \sqrt{\frac{C(h)\log(2) + 2\log(C(h)) + \log(2/\delta)}{2m}}\bigg] \ge 1 - \delta
\end{align*}
where $\{x_i\}_{i=1}^m \overset{\mathrm{iid}}{\sim} \mathcal{P}_0$, $\{\xi_{i,j}\}_{i \in [m],\, j \in [n]} 
\overset{\mathrm{iid}}{\sim} \mathcal{P}_\xi$, and $C(h)$ is the 
compression size in bits of $h$ using a pre-defined encoding scheme.
\end{theorem}

\begin{remark}[Error decomposition]
\label{rem:error-decomposition}
Theorem~\ref{theorem:bound} directly follows from a decomposition of the 
generalization gap into two components, each addressing one source of 
randomness:
\begin{equation*}
    \mathcal{R}(h) - \hat{R}_{m,n}(h) 
    \;=\; \underbrace{\mathcal{R}(h) - \frac{1}{m}\sum_{i=1}^{m} \bar{r}(x_i; h)}
        _{\varepsilon_{\mathrm{ext}}(h)\; \text{(extrinsic error)}}
    \;+\; \underbrace{\frac{1}{m}\sum_{i=1}^{m} \bar{r}(x_i; h) - \hat{R}_{m,n}(h)}
        _{\varepsilon_{\mathrm{int}}(h)\; \text{(intrinsic error)}},
\end{equation*}
where $\bar{r}(x; h) = \mathbb{E}_{\xi}[r(x, g(x, \xi; h))]$ and 
$\hat{R}_{m,n}(h) = \frac{1}{mn}\sum_{i,j} r(x_i, g(x_i, \xi_{i,j}; h))$. 
The extrinsic error $\varepsilon_{\mathrm{ext}}(h)$ measures the generalization 
over the prompt distribution, while the intrinsic error 
$\varepsilon_{\mathrm{int}}(h)$ measures estimation accuracy of the empirical 
reward given $n$ sampled generations per prompt. Bounding each yields the 
$\mathcal{O}(1/\sqrt{m})$ and $\mathcal{O}(1/\sqrt{mn})$ rates in 
Theorem~\ref{theorem:bound}. Because $m$ is fixed by the training set, while $n$ 
is freely chosen at evaluation time, taking $n = \Omega(m)$ shrinks the 
intrinsic rate to $\mathcal{O}(1/m)$, lower-order than the extrinsic 
$\mathcal{O}(1/\sqrt{m})$. Hence, the bound is effectively governed by the 
extrinsic error.
\end{remark}

\subsection{Extensions of Theorem \ref{theorem:bound}}
\label{sec:multi-turn}

We extend Theorem~\ref{theorem:bound} to two practically important
settings: (1) multi-turn RLVR, where the model interacts with an environment
over multiple steps, and (2) out-of-distribution (OOD) evaluation. Multi-turn
RLVR introduces a potential new source of randomness from the environment's 
responses on top of the per-token Gumbel noise treated in 
Section~\ref{sec:rlvr-bounds}. We handle multi-turn RLVR by considering two 
cases: environments whose transitions are deterministic, and environments whose 
stochastic transitions admit a reparameterization.

\minihead{Multi-turn RLVR with deterministic environments.}
Theorem~\ref{theorem:bound} extends without modification to multi-turn RLVR with 
deterministic environments, which cover most practical RLVR settings, such as 
mathematical problem-solving with symbolic verifiers and code execution on 
deterministic tests. In such settings, the environment introduces no additional 
randomness due to deterministic transition and termination functions. Given a 
prompt and interaction history, the next observation and termination signal are 
uniquely determined. The full trajectory $\tau = (o_0, a_1, o_1, \dots, a_H, o_H)$ 
over a bounded horizon $H$ is therefore a deterministic function of the prompt, 
the model's responses, and the environment's responses. Because the model's 
responses are themselves token sequences generated via Gumbel-max 
reparameterization, the trajectory reduces to a deterministic function 
$G(x_i, \xi_{i,j}; h)$, where $\xi_{i,j} \in \mathbb{R}^{H \times T \times V}$ 
collects noise across all $H$ turns. The proof of Theorem~\ref{theorem:bound} 
applies verbatim under the substitution $g \mapsto G$.

\minihead{Multi-turn RLVR with stochastic environments.}
When the environment itself is \emph{stochastic}, its transitions contribute a
second source of randomness, independent of the model's token decoding. If 
stochastic transitions of the environment admit a reparameterization (e.g., tool 
responses expressed as deterministic functions of random noise), we introduce a 
hypothesis-independent noise variable $\eta_{i,j} \sim \mathcal{P}_\eta$ 
alongside the Gumbel noise. The trajectory becomes 
$G(x_i, \xi_{i,j}, \eta_{i,j}; h)$, and Theorem~\ref{theorem:bound} continues to 
hold with $(\xi, \eta)$ jointly playing the role of the noise variable.
Non-stationary environments, such as live tool APIs \cite{wang2026openclaw} or 
environments whose transitions shift between training and deployment,
fall outside the scope of this work and are left for future work.

\minihead{OOD evaluation.}
When a labeled corpus is available from a different target distribution 
$\mathcal{Q}_0$, a similar bound applies to the expected reward 
$\mathcal{R}_{\mathcal{Q}}(h)$ under $\mathcal{Q}_0$. As long as the hypothesis 
is chosen independently of the OOD corpus, the compression penalty 
$C(h)\log 2 + 2\log C(h)$ drops out and only a Hoeffding term of size 
$\mathcal{O}(\Delta/\sqrt{m'})$ in the OOD sample size $m'$ remains. We defer 
the formal statement (Theorem~\ref{theorem:ood-bound}) and proof to
Appendix~\ref{subsec:app-ood-bound}, and an empirical study to 
Appendix~\ref{sec:app-ood-evals}.

\subsection{Subsampling for Efficient Empirical Reward Computation}
\label{sec:subsample-bounds}

The generalization error bound in Theorem~\ref{theorem:bound} involves 
calculating the empirical reward $\frac{1}{mn}\sum_{i=1}^{m}\sum_{j=1}^n r(x_i, g(x_i, \xi_{i,j}; h))$, 
which requires evaluating the RLVR model on all $m$ data points and 
$n$ noise vectors per data point. When $m$ is large, this full-dataset evaluation is 
computationally expensive. With the following corollary, we can estimate the 
empirical reward using a subsample of training prompts, incurring a small 
additive penalty. We defer the proof to Appendix \ref{subsec:subsampling-proof}.

\begin{corollary}
\label{cor:subsampled-eval}
Under the setting of Theorem~\ref{theorem:bound}, with training prompts
$x_1, \dots, x_m \overset{\mathrm{iid}}{\sim} \mathcal{P}_0$ and independent
Gumbel noise vectors $\xi_{i,j} \overset{\mathrm{iid}}{\sim} \mathcal{P}_\xi$ 
for $i \in [m],\, j \in [n]$, let $\{i_1, \dots, i_{m_0}\} \subset [m]$ be a 
subsample drawn uniformly at random, and $h \in \mathcal{H}$ be 
any hypothesis chosen independently of the subsample.
Then for any $\delta \in (0, 1)$,
\begin{align*}
    \mathbb{P}\bigg[
    \mathcal{R}(h) \,\ge\,
    & \frac{1}{m_0 n} \sum_{k=1}^{m_0} \sum_{j=1}^{n}
        r\!\bigl(x_{i_k},\, g(x_{i_k}, \xi_{i_k, j};\, h)\bigr)
    \;-\; \Delta \sqrt{\tfrac{\log(3/\delta)}{2 m_0}} \\
    & \;-\; \Delta\!\left(\sqrt{\tfrac{1}{n}} + 1\right)
        \sqrt{\tfrac{C(h)\log 2 + 2\log C(h) + \log(3/\delta)}{2m}}
    \bigg] \,\ge\, 1 - \delta.
\end{align*}
\end{corollary}

\section{Progressive RLVR via On-Policy Distillation}
\label{sec:progressive-RLVR}

\begin{figure}[t]
    \centering
    \includegraphics[width=\textwidth]{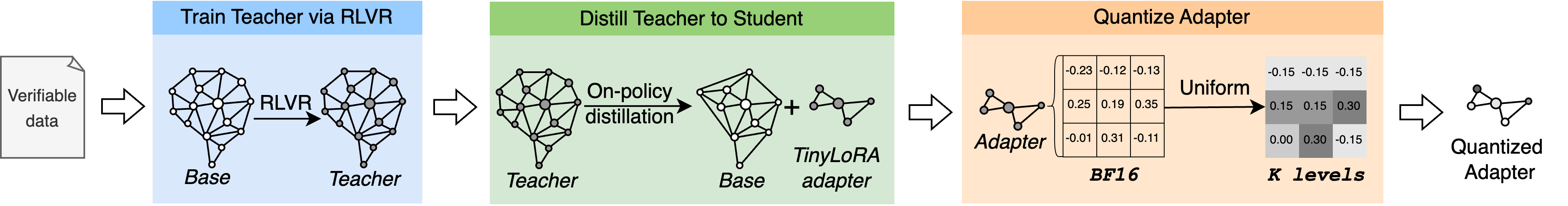}
    \caption{Our proposed Progressive RLVR pipeline, which integrates RLVR with 
    parameter-efficient fine-tuning, on-policy distillation, and quantization to 
    achieve high training rewards while preserving small compression sizes.}
    \label{fig:progressive-rlvr}
\end{figure}

To establish non-vacuous bounds via Theorem~\ref{theorem:bound}, we require 
simultaneously large empirical reward $r(x_i, g(x_i, \xi_{i,j}; h))$ and small 
compression size $C(h)$. While this is 
challenging, the PAC-Bayes bounds in Theorem~\ref{theorem:bound} do not 
constrain the procedure that produces $h$. This decouples \emph{capacity used 
during search} from \emph{capacity charged to the bound}. Namely, we are free to
use a high-capacity teacher to explore the reasoning landscape, as long as the 
final compression size is small. To exploit this asymmetry, we introduce the
Progressive RLVR pipeline (Figure~\ref{fig:progressive-rlvr}).

\minihead{RLVR and distillation}. As shown in 
Figure~\ref{fig:progressive-rlvr}, we first train a high-capacity teacher model 
using a standard RLVR objective. This allows the teacher to freely explore and 
master complex 
\begin{wrapfigure}{r}{0.3\textwidth}
    \centering
    \includegraphics[width=0.3\textwidth]{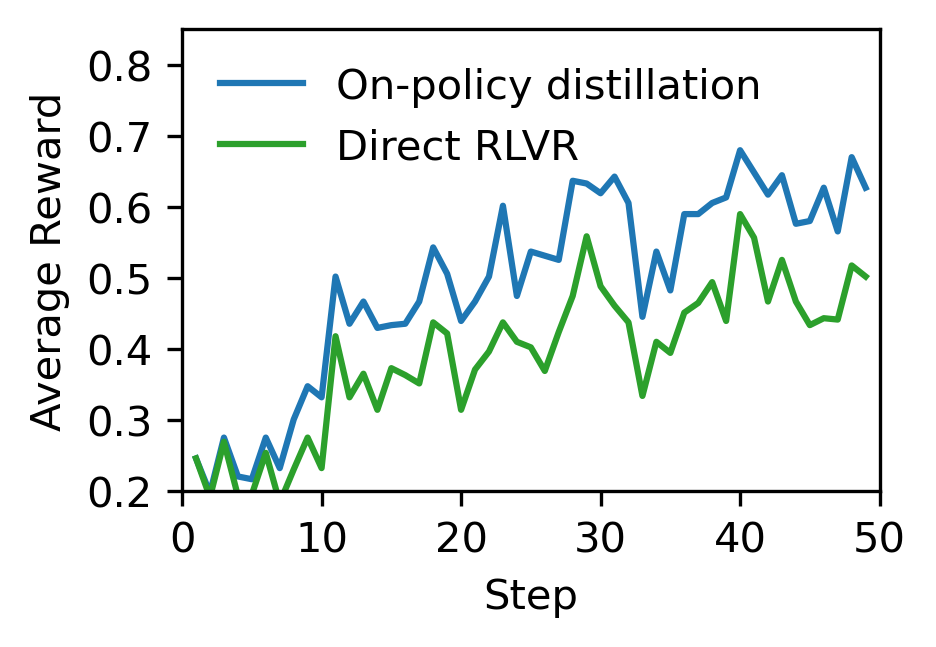}
    \caption{Training rewards of the distillation phase of 
    Progressive RLVR, achieving 10\% higher rewards than direct RLVR on
    average.}
    \label{fig:distill-vs-rlvr}
\end{wrapfigure}
reasoning trajectories via RLVR without restrictive parameter bottlenecks. 
Enabling this exploration is a key feature of our approach, since the bound 
depends only on the final model, not on the procedure. Next, 
we apply on-policy distillation \cite{agarwal2024policy} to compress the 
capabilities learned in RLVR into a student model with TinyLoRA 
\cite{morris2026learning} parameterization. On-policy distillation improves the 
sample efficiency and boosts the performance of ultra-compact models by 
providing token-level optimization signals. In Figure~\ref{fig:distill-vs-rlvr}, 
we show that on the math domain, on-policy distillation achieves 10\% higher 
rewards than direct RLVR on average, given the same number of training steps.

\minihead{Quantization}. The final stage of our Progressive RLVR pipeline 
applies quantization to the trained TinyLoRA adapter of the student 
model to further reduce its description length. Specifically, we quantize the 
parameters of the TinyLoRA adapter into uniformly distributed discrete levels
\cite{lotfi2023non,han2015deep}. Empirically, this 
quantization step vastly reduces the bit-width of the adapter while inducing 
only minor degradation in the acquired empirical reward.

\section{Non-Vacuous RLVR Generalization Bounds for Billion-Parameter LLMs}
\label{sec:results}
In this section, we empirically demonstrate that our framework yields
non-vacuous and tight generalization bounds for billion-parameter LLMs
post-trained via TinyLoRA-based RLVR. We first detail our experimental setup, including
environments, models, and optimization hyperparameters in Section 
\ref{sec:exp-setting}. We then evaluate our bounds and report the main empirical 
findings in Section \ref{sec:non-vacuous}.

\subsection{Experimental Settings}
\label{sec:exp-setting}

\minihead{Datasets and environments.}
To demonstrate the applicability of our framework on diverse application 
domains, we consider four datasets and environments:
\vspace{-0.5em}
\begin{enumerate}[leftmargin=*]
    \item \textbf{Mathematical problem-solving}: We use the mathematics subset 
    of the Eurus-2-RL dataset \cite{cui2025process}, sourced from NuminaMath-CoT 
    \cite{numina_math_datasets}. This contains 456,285 questions ranging from 
    high-school mathematics to International Mathematical Olympiad challenges. 
    We standardize multiple-choice questions to require 
    only the choice index and normalize string representations (\eg, unicode) 
    to LaTeX. We verify the correctness of an answer using 
    \texttt{sympy} and the \texttt{math\_verify}\footnote{\url{https://github.com/huggingface/Math-Verify}} library, yielding a binary reward $r \in \{0, 1\}$. Hence, the average 
    reward is equivalent to task accuracy.
    \item \textbf{Programming}: We construct a programming dataset 
    by merging the programming subset of Eurus-2-RL \cite{cui2025process} with the 
    KodCode-V1 dataset \cite{xu2025kodcode}, yielding 504,885 programming tasks. 
    We evaluate correctness of a generated program against the unit tests 
    provided with each task, yielding a binary reward $r \in \{0, 1\}$.
    \item \textbf{General knowledge}: We combine large-scale datasets 
    from knowledge-intensive domains, including Med-QA \cite{jin2020disease}, 
    LogiQA \cite{liu2020logiqa}, C-Eval \cite{huang2023ceval}, Arc \cite{allenai:arc}, 
    LegalBench \cite{guha2023legalbench,koreeda2021contractnli,hendrycks2021cuad,
    wang2023maud,wilson2016creation,zheng2021does,zimmeck2019maps,
    ravichander2019question,holzenberger2021factoring,lippi2019claudette}, MMLU 
    \cite{hendryckstest2021,hendrycks2021ethics}, and Webinstruct-verified 
    \cite{ma2025generalreasoner}. We exclude the ``Mathematics'' subset of the 
    Webinstruct-verified dataset to avoid overlap with the math domain, yielding 
    438,908 data points. We consider an answer correct if the final answer is 
    wrapped by \texttt{\textbackslash boxed\{\}} and matches 
    the ground truth exactly, leading to a binary reward $r \in \{0, 1\}$.
    \item \textbf{Text-to-SQL}: We use the SynSQL-2.5M dataset \cite{li2025omnisql}, which 
    contains 2,544,390 synthetic natural language queries mapped to SQL queries 
    across 16,583 relational databases. Following SkyRL-SQL \cite{liuskyrl}, we 
    enforce a strict output formatting constraint (\texttt{<solution></solution>}) 
    to reliably extract the generated SQL. Our environment allows the model to 
    interactively query the database (e.g., issuing exploratory schema queries) 
    for up to five turns. We assign a trajectory a reward of 1 if the final 
    query returns the same result as the ground-truth SQL query, 0 if it leads to 
    a different result, and -1 if it is malformed or unparseable. Although 
    training uses a tri-level reward, we report the bound on binary task 
    accuracy to be consistent with the other domains. As discussed in 
    Section~\ref{sec:rlvr-bounds}, our results remain valid because the 
    PAC-Bayes bound depends on the hypothesis only through its prior and 
    empirical reward, not on the training objective.  We defer the non-vacuous 
    generalization bound on the tri-level reward to Appendix \ref{sec:app-sql-reward}.
\end{enumerate}

\begin{figure}
    \centering
    \includegraphics[width=\linewidth]{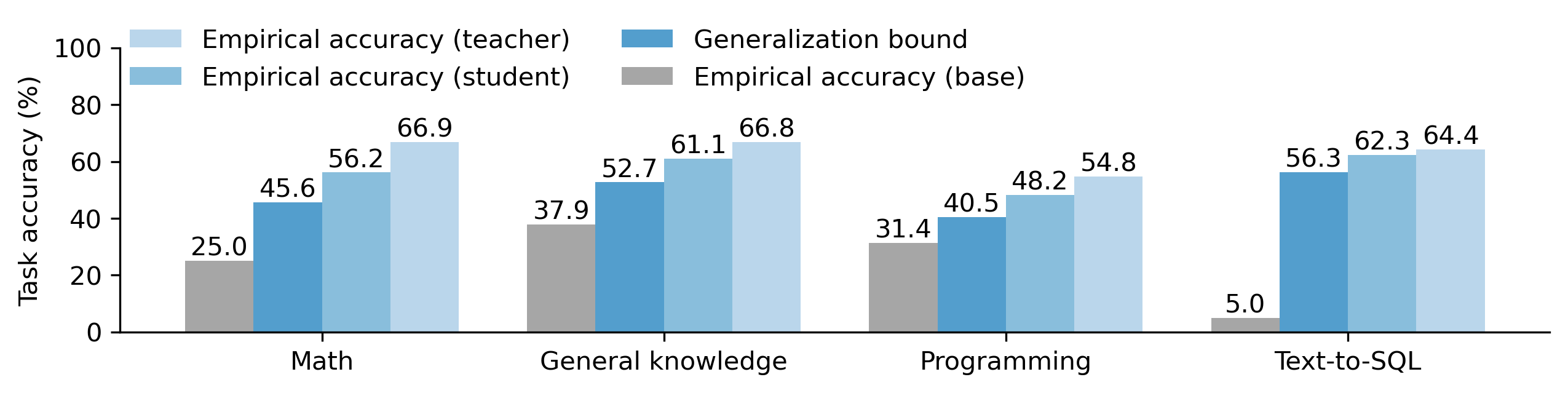}
    \caption{Our framework yields non-vacuous generalization bounds on billion-parameter LLMs 
    post-trained via Progressive RLVR on four domains. We report the 
    generalization bound ($\delta = 0.05$) on the quantized TinyLoRA student, 
    alongside the empirical accuracy of the teacher, student, and base model. 
    The bound exceeds the base model's accuracy in all four domains and lies 
    within \ndiffstudent of the student's accuracy.}
    \label{fig:overall-bounds}
\end{figure}

\minihead{Models.}
We use the pre-trained Qwen3.5-4B \cite{qwen35blog} as the base model 
($\pi_{\theta_{\text{ref}}}$) for both the teacher and student. 
We additionally use Qwen3.5-2B for ablation studies.

\minihead{Hyperparameters.}
To train teachers, we apply standard LoRA with a rank of 64. We use a standard 
importance-sampling objective \cite{shao2024deepseekmath,liu2025understanding} 
without a KL-divergence penalty. For the on-policy distillation, we use TinyLoRA 
with an inner dimension of $u=16$ and a projection rank of $r=16$. In both 
phases, we use a batch size of 64 and 8 rollouts per prompt. We sweep over the 
learning rates $\{10^{-5}, 5\times 10^{-5}, 10^{-4}\}$, selecting the 
configuration that maximizes the training reward. For reward estimation 
(Corollary~\ref{cor:subsampled-eval}), we sample 4,096 data points and 64 
noise vectors per data point. For quantization, we use five levels (2.3 bits) by 
default. We defer more details about the hyperparameters to Appendix 
\ref{appendix:exp_setup} and the encoding scheme to Appendix \ref{sec:app-encoding}.

\minihead{Implementation and hardware.}
We implement the teacher training procedure using Tinker 
API\footnote{\url{https://tinker-docs.thinkingmachines.ai/tinker}} and execute 
it on a local cluster of 8$\times$3GHz vCPUs with 80 GB RAM. We implement 
on-policy distillation with TinyLoRA via SkyRL \cite{cao2025skyrl} and execute 
it with two A100s and 512 GB of RAM.

\subsection{Non-vacuous Bounds on Real-World RLVR Tasks}
\label{sec:non-vacuous}

\begin{figure}
    \centering
    \begin{minipage}{0.49\textwidth}
        \centering
        \includegraphics[width=\textwidth]{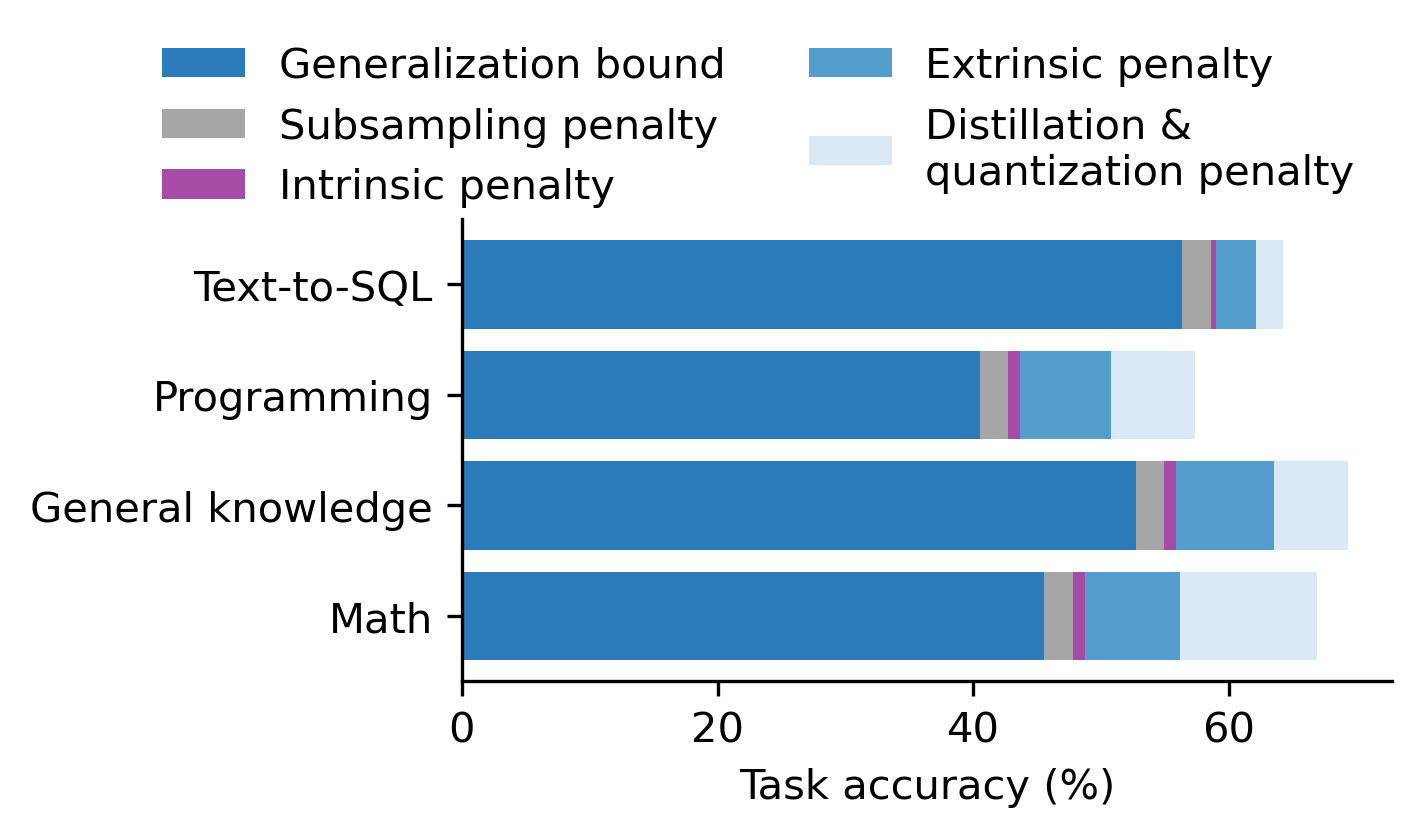}
        \caption{Decomposition of the teacher's accuracy (the total bar 
        length) into our generalization bounds and the penalties separating them.}
        \label{fig:decomp}
    \end{minipage}
    \hfill
    \begin{minipage}{0.45\textwidth}
        \centering
        \includegraphics[width=\textwidth]{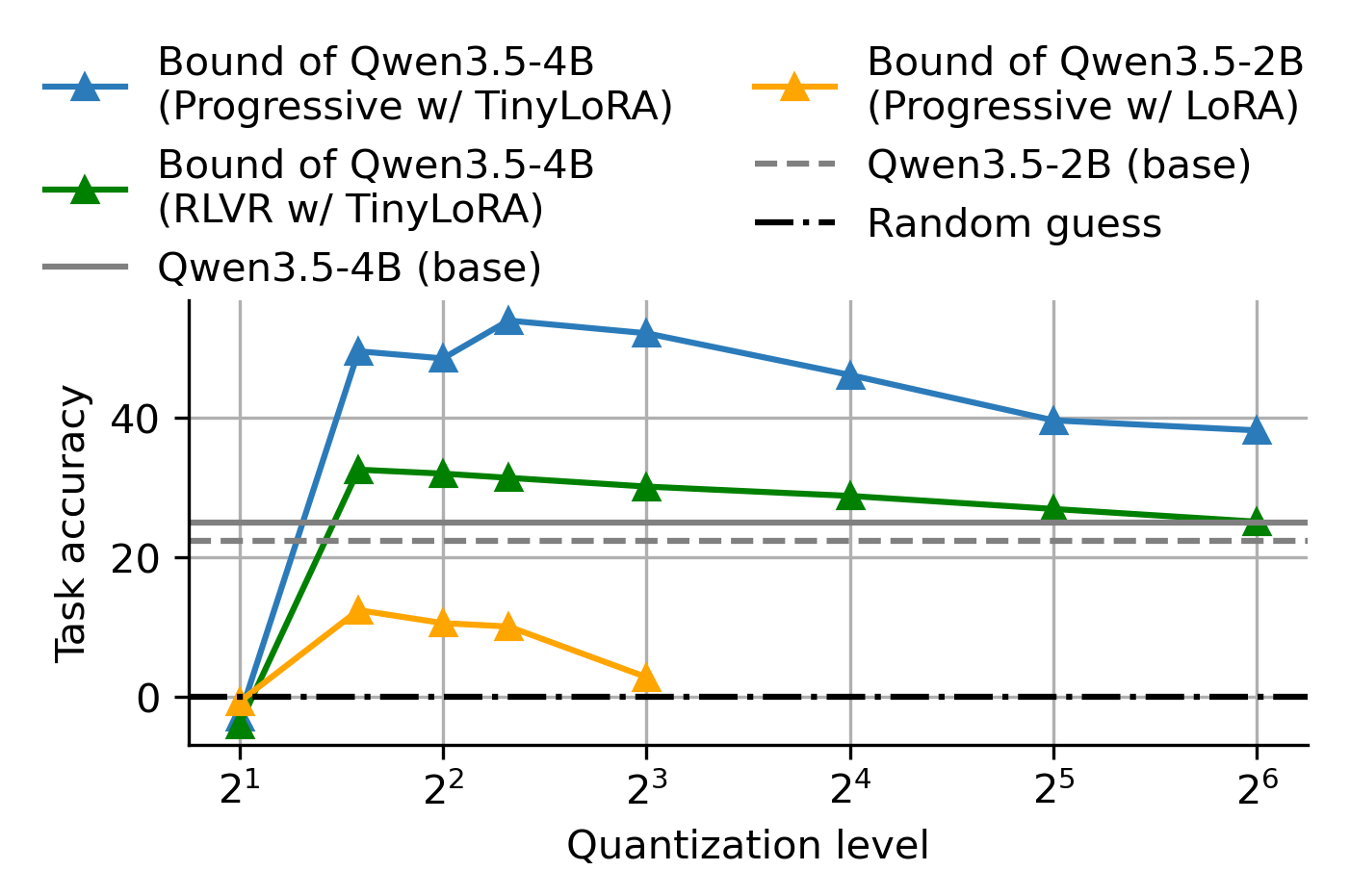}
        \caption{Ablation of the Progressive RLVR pipeline. All bounds are 
        computed at $\delta = 0.05$. Bounds below zero are vacuous.}
        \label{fig:baselines}
    \end{minipage}
\end{figure}

\minihead{Non-vacuous bounds across four domains.} 
In Figure~\ref{fig:overall-bounds}, we present our primary results, the 
generalization bounds from Corollary~\ref{cor:subsampled-eval}. As shown, the 
generalization bounds are non-vacuous (outperforming both random guess and the 
base model) in all domains, indicating that the model fine-tuned via Progressive 
RLVR endows the base model with new capabilities that generalize beyond the 
training data. Specifically, our bounds exceed the accuracy of the base model by 
\ndiffbase, providing meaningful verification for the generalizability of 
billion-parameter LLMs post-trained via Progressive RLVR. No prior framework 
produces non-vacuous generalization bounds for billion-parameter RLVR models on 
real-world datasets.

In addition to being non-vacuous, our bounds are tight. As shown, the gap 
between the student's accuracy and the generalization bound ranges over 
\ndiffstudent across domains. This consistency across different tasks suggests 
that the tightness of our bounds is governed by properties of the training 
setup (sample sizes, compression) rather than domain-specific factors.

\minihead{Decomposition of the bounds.}
In Figure~\ref{fig:decomp}, we break down the accuracy of the teacher into our 
generalization bound and the penalties separating the bound from the teacher's 
accuracy. As shown, the dominant 
penalty differs across domains. For Text-to-SQL, all penalties are small 
(total 8\%), reflecting both the student's close match to the teacher and the 
tight extrinsic penalty due to a large training set of 2.5 million examples. 
For other domains, we identify similar contributions from the extrinsic 
penalty (7--8\%) and the distillation and quantization penalty (6--11\%), 
suggesting that additional training prompts and a better optimized TinyLoRA 
configuration would tighten these bounds. In addition, we find that the intrinsic 
generalization gaps are negligible in all domains, confirming that 64 generations 
per prompt are sufficient for reward estimation.

\minihead{Each component of Progressive RLVR is necessary.} 
In Figure~\ref{fig:baselines}, we ablate two ingredients of 
Progressive RLVR: (i) on-policy distillation (removed in the ``RLVR w/ 
TinyLoRA'') and (ii) TinyLoRA (replaced with standard LoRA, $r=1$). Since 
standard LoRA applied directly to Qwen3.5-4B yields vacuous bounds across all 
quantization levels due to the LoRA adapter's substantially larger parameter 
count, we omit this curve from the figure. We additionally show a baseline using 
a Qwen3.5-2B and standard LoRA with a rank of 1 as the student.

We find that removing on-policy distillation reduces the bound by \ntightopd 
on average relative to our progressive method, confirming that the dense 
token-level supervision from distillation is essential for training the compact 
student under the parameter budget our bound requires. Furthermore, the bounds 
via standard LoRA with Qwen3.5-2B are lower than the base 2B model's accuracy
and slightly higher than random guess, showing that a smaller model 
is not a substitute for Progressive RLVR.


\section{Discussion}
\label{sec:discussion}
\minihead{Toward verifiable specialization.} Our work takes a first step toward
a deployment paradigm we call \emph{verified specialization}: training a
domain-specific RLVR model and obtaining, before deployment, a high-probability
lower bound on its expected accuracy on unseen data drawn from the same
distribution as the training set. This vision is increasingly relevant as
practitioners specialize open-weight models on proprietary or internal data and
find them outperforming larger general-purpose models on target tasks 
\cite{databricks-RLVR}. Without theoretical verification, it remains unclear 
whether the RLVR model overfits or generalizes. Our bounds 
for Qwen3.5-4B exceed the base model's accuracy by \ndiffbase, showing that a
meaningful verification of this kind is achievable at billion-parameter scale. 
Beyond establishing this in theory, our proposed Progressive RLVR offers a 
concrete recipe for achieving it in practice: train a high-capacity teacher, 
distill into a compact student, and quantize. The effectiveness of this recipe 
further suggests that distillation plays a role beyond cost reduction, enabling 
verifiable generalization that pipelines relying on compression alone struggle 
to achieve.

\minihead{Limitations and future directions.} 
Our work has two limitations that could motivate directions for future research. 
First, our bounds apply only when the deployment distribution matches the training 
distribution. They fail to hold under domain shift without labeled out-of-domain 
data. Second, our experiments are at the 4B-parameter scale, while real-world 
deployed models (\eg, Kimi-K2.6 \cite{kimi26-blog} and Qwen3.5-397B-A17B 
\cite{qwen35blog}) are often orders of magnitude larger. The compression budget 
our bounds require has not been shown achievable at that scale without 
sacrificing empirical reward. Therefore, we call for future study on the 
generalizability of larger-scale models in practical use.

\section{Related Work}
\label{sec:related-work}

\minihead{RLVR and its generalizability.}
Since the introduction of GRPO \cite{shao2024deepseekmath}, subsequent work has 
proposed various algorithmic improvements to address stability and 
optimization biases in RLVR. DAPO \cite{yu2025dapo} removes the KL penalty and 
introduces dynamic clipping bounds to encourage exploration. Similarly, CISPO 
\cite{chen2025minimax} and SAPO \cite{gao2025soft} utilize asymmetric clipping 
to maintain stable entropy. Dr.\xspace GRPO mitigates the optimization bias that 
inflates incorrect response lengths \cite{liu2025understanding}, while TIS 
resolves the systematic divergence between sampling and learning in 
decoupled RL training systems \cite{yao2025offpolicy}.

Despite these algorithmic successes, a tension exists in the empirical 
literature regarding RLVR's generalizability. Some studies show that RLVR's 
generalization to unseen data is fragile \cite{hu2025breaking}. RLVR 
often induces narrowed exploration that forces models to overfit to highly 
rewarded training trajectories \cite{yue2025does,nguyen2025reasoning}. 
Conversely, other comparative studies argue that RL post-training fosters 
broader generalization than standard supervised fine-tuning \cite{chu2025sft}. 
This empirical debate underscores the critical need for rigorous theoretical 
generalization bounds.

\minihead{Generalization Bounds.}
Establishing generalization bounds for deep neural networks is a central 
challenge in learning theory. Classic complexity measures, such as 
VC dimension \cite{blumer1989learnability} and Rademacher complexity 
\cite{bartlett2002rademacher}, scale with parameter counts and yield vacuous 
bounds for modern models \cite{zhang2016understanding,nagarajan2019uniform}. 
While norm-based methods attempt to explain generalization via small weight 
norms \cite{neyshabur2015norm,golowich2018size}, the large spectral norms of 
practical networks often make these bounds similarly vacuous 
\cite{arora2018stronger}.

PAC-Bayes \cite{lotfi2022pac,grunwald2019tight} has emerged as an
important theory to explain the generalization of deep neural networks. Prior 
work has established non-vacuous compression-based PAC-Bayes bounds for LLM 
pre-training \cite{lotfi2023non}. However, extending these bounds to 
RLVR is non-trivial. Reinforcement learning optimizes over a non-stationary 
state-action distribution dependent on the active policy, violating the 
\emph{i.i.d.} assumptions of standard supervised bounds. While prior work 
\cite{wang2019generalization} established bounds for specific reparameterizable 
RL, their results are limited to simulation with small-scale networks. 

\minihead{Parameter-Efficient Fine-Tuning.}
To reduce the massive computational costs of fine-tuning LLMs, a wide range of 
PEFT techniques has emerged \cite{wang2025parameter,xu2026parameter}. LoRA is 
one of the most successful PEFT techniques for fine-tuning LLMs
\cite{hu2022lora,schulman2025lora}. Beyond the standard approach 
\cite{hu2022lora}, recent work on LoRA focuses on structural optimization and extreme compression. For instance, PLoP 
\cite{hayou2025plop} automates adapter placement by targeting modules with low 
feature norms. Other methods structurally redefine the parameter space: 
Uni-LoRA \cite{li2025uni} flattens parameters into high-dimensional vectors, 
CrossSpectra \cite{zhang2025crossspectra} parameterizes global adaptations via 
sparse spectral coefficients, and TRAC \cite{ye2025trac} exploits layer 
redundancy via tensor-train cores. Further innovations include principal column 
gradient projections \cite{hwang2025pica}, and quantization frameworks 
\cite{li2023loftq,Xu2023qalora}. While these methods address practical memory 
constraints, our framework specifically leverages the TinyLoRA approach 
\cite{morris2026learning} for achieving the aggressive compression required by 
PAC-Bayes compression bounds.

\section{Conclusion}

We present the first non-vacuous generalization bounds for LLMs post-trained
via RLVR. Our approach leverages PAC-Bayes compression bounds and uses a 
Gumbel-max reparameterization to account for the stochasticity of token 
generation and the combinatorially large support of the language model's output 
distribution. To meet the compression size and accuracy requirements of the 
theoretical bounds, we introduce the Progressive RLVR pipeline, which integrates 
RLVR with TinyLoRA, on-policy distillation, and quantization. Empirically, we 
demonstrate non-vacuous bounds for billion-parameter LLMs across four real-world 
application domains, with Progressive RLVR achieving tighter bounds than direct
RLVR with TinyLoRA. We hope this work provides a foundation for verifying the 
generalization of RLVR models, particularly in high-stakes applications.

\begin{ack}
We thank John Schulman for valuable feedback on this work. This research was 
supported in part by a sponsorship from Bridgewater AIA Labs. We also gratefully 
acknowledge CloudLab \cite{Duplyakin+:ATC19} for providing the computational 
resources used to conduct the experiments in this work.
\end{ack}


\bibliographystyle{plainnat}
\bibliography{reference}







\appendix
\section{Proof of Generalization Bounds}
\label{app:proof}

\subsection{Notation}

We summarize the notation in Table \ref{tab:notation}.

\begin{table}[t]
\centering
\caption{Summary of notation.}
\label{tab:notation}
\begin{tabular}{@{}cl@{}}
\toprule
\textbf{Symbol} & \textbf{Description} \\
\midrule
\multicolumn{2}{@{}l}{\textit{Spaces and distributions}} \\
$\mathcal{X}$ & Prompt space \\
$\mathcal{Y}$ & Generation (sequence) space \\
$\mathcal{P}_0$ & Population distribution over task prompts \\
$\hat{\mathcal{P}}_0$ & Empirical distribution over $m$ training prompts \\
$\mathcal{P}_\xi$ & Product distribution over Gumbel noise vectors \\
$\mathcal{H}$ & Finite hypothesis space \\
\midrule
\multicolumn{2}{@{}l}{\textit{Models and hypotheses}} \\
$\theta_{\mathrm{ref}}$ & Parameters of the frozen pre-trained base model \\
$\pi_{\theta_{\mathrm{ref}}}$ & Base model policy \\
$h$ & Hypothesis (parameter update) in $\mathcal{H}$ \\
$\pi_{\theta(h)}$ & Fine-tuned model policy under hypothesis $h$ \\
\midrule
\multicolumn{2}{@{}l}{\textit{Generation and reward}} \\
$V$ & Vocabulary size \\
$T$ & Maximum sequence length \\
$\tau$ & Sampling temperature \\
$\ell_{i,t} \in \mathbb{R}^V$ & Logits at step $t$ for prompt $x_i$ \\
$\xi_{i,j} \in \mathbb{R}^{T \times V}$ & Gumbel noise for prompt $i$, generation $j$ \\
$g(x_i, \xi_{i,j};\, h)$ & Deterministic generation function (reparameterized) \\
$r \colon \mathcal{X} \times \mathcal{Y} \to [a,\, a{+}\Delta]$ & Deterministic, bounded reward function \\
$\Delta$ & Range of the reward function \\
$\bar{r}(x; h)$ & Expected reward: $\mathbb{E}_{\xi}[r(x, g(x, \xi; h))]$ \\
$\mathcal{R}(h)$ & Population reward of hypothesis $h$ \\
\midrule
\multicolumn{2}{@{}l}{\textit{Sampling and evaluation}} \\
$m$ & Number of training prompts \\
$n$ & Number of generations (noise draws) per prompt \\
$m_0$ & Number of subsampled prompts ($m_0 \le m$) \\
\midrule
\multicolumn{2}{@{}l}{\textit{Complexity and prior}} \\
$K(h)$ & Prefix Kolmogorov complexity of $h$ \\
$C(h)$ & Compression size of $h$ in bits (upper bound on $K(h)$) \\
$P(h)$ & Solomonoff prior: $Z^{-1} \cdot 2^{-K(h)}$ \\
$Z$ & Normalizing constant ($Z \le 1$) \\
\midrule
\multicolumn{2}{@{}l}{\textit{Error decomposition}} \\
$\varepsilon_{\mathrm{ext}}(h)$ & Extrinsic error (prompt generalization gap) \\
$\varepsilon_{\mathrm{int}}(h)$ & Intrinsic error (generation Monte Carlo gap) \\
$\varepsilon_{\mathrm{sub}}(h)$ & Subsampling error (Corollary~\ref{cor:subsampled-eval}) \\
$\delta$ & Confidence parameter \\
\bottomrule
\end{tabular}
\end{table}

\subsection{Proof of Theorem \ref{theorem:bound}}
\label{subsec:app-bound-proof}



\begin{proof}

We decompose the generalization gap into two components:
\begin{align*}
& \mathcal{R}(h) 
    - \frac{1}{mn} \sum_{i=1}^{m} \sum_{j=1}^n r\!\bigl(x_i, g(x_i, \xi_j; h)\bigr) \\
&= \underbrace{\mathbb{E}_{x \sim \mathcal{P}_0} 
    \!\left[
        \mathbb{E}_{\xi \sim \mathcal{P}_{\xi}}
        \!\left[ r(x, g(x, \xi; h)) \right]
    \right] 
    - \frac{1}{m}\sum_{i=1}^{m} 
      \mathbb{E}_{\xi \sim \mathcal{P}_{\xi}}
      \!\left[ r(x_i, g(x_i, \xi; h)) \right]}_{\textstyle 
      \varepsilon_{\mathrm{ext}}(h) \;(\text{extrinsic error})} \\[6pt]
&\quad+ \underbrace{\frac{1}{m}\sum_{i=1}^{m}
      \mathbb{E}_{\xi \sim \mathcal{P}_{\xi}}
      \!\left[ r(x, g(x, \xi; h)) \right]
    - \frac{1}{mn} \sum_{i=1}^{m} \sum_{j=1}^n r\!\bigl(x_i, g(x_i, \xi_j; h)\bigr)}_{\textstyle 
      \varepsilon_{\mathrm{int}}(h) \;(\text{intrinsic error})}
\end{align*}
We bound each term separately.

\minihead{Bounding the extrinsic error.}
Recall that $\hat{\mathcal{P}}_0$ is the uniform distribution over $m$ training 
prompts drawn \iid from $\mathcal{P}_0$. Define 
$\bar{r}(x; h) \coloneqq \mathbb{E}_{\xi \sim \mathcal{P}_{\xi}}
[r(x, g(x, \xi; h))]$. Since $r \in [a, a+\Delta]$, the function $\bar{r}$ 
is bounded in the same interval. The extrinsic error is then
\begin{equation*}
    \varepsilon_{\mathrm{ext}}(h) 
    = \mathbb{E}_{x \sim \mathcal{P}_0}[\bar{r}(x; h)] 
    - \frac{1}{m}\sum_{i=1}^{m} \bar{r}(x_i; h),
\end{equation*}
which is the deviation of the sample mean of $m$ \iid bounded random variables 
from their population mean. By Hoeffding's inequality, for any $t_1 > 0$,
\begin{equation*}
    \mathbb{P}\!\left[\varepsilon_{\mathrm{ext}}(h) \ge -t_1 \right] 
    \ge 1 - \exp\!\left(-\frac{2m\, t_1^2}{\Delta^2}\right).
\end{equation*}

To ensure this bound holds uniformly over all $h \in \mathcal{H}$, we distribute 
half of the total confidence budget, $\delta/2$, across hypotheses using the 
Solomonoff prior, allocating a failure probability of 
$\tfrac{1}{2}\delta \cdot P(h)$ to each $h$. Setting the Hoeffding tail equal 
to this budget and applying the union bound yields
\begin{equation*}
    \exp\!\left(-\frac{2m\, t_{1,h}^2}{\Delta^2}\right) 
    = \tfrac{1}{2}\delta \cdot P(h).
\end{equation*}
Solving for $t_{1,h}$:
\begin{equation*}
    t_{1,h} = \Delta\sqrt{\frac{\log(1/P(h)) + \log(2/\delta)}{2m}}.
\end{equation*}
Since $\sum_{h \in \mathcal{H}} P(h) \le 1$, the union bound gives
\begin{equation*}
    \mathbb{P}\!\left[
      \forall h \in \mathcal{H},\; 
      \varepsilon_{\mathrm{ext}}(h) \ge - t_{1,h}
    \right]
    \ge 1 - \tfrac{1}{2}\delta.
\end{equation*}

\minihead{Bounding the intrinsic error.}
Conditional on the training prompts $x_1, \dots, x_m$, the intrinsic error 
measures the deviation of a Monte Carlo estimate from its expectation. Because 
the noise vectors $\xi_{i,j}$ are drawn independently for each prompt, the $mn$ random variables 
$r(x_i, g(x_i, \xi_{i,j}; h))$ are mutually independent and each bounded in 
$[a, a + \Delta]$. Note that Hoeffding's inequality requires only independence 
and uniform boundedness, not identical distributions --- the reward distributions 
may differ across prompts. Applying Hoeffding's inequality over $mn$ independent 
samples and distributing the remaining budget $\delta/2$ via the same 
prior-weighted union bound, we obtain
\begin{equation*}
    t_{2,h} = \Delta\sqrt{\frac{\log(1/P(h)) + \log(2/\delta)}{2mn}},
\end{equation*}
with
\begin{equation*}
    \mathbb{P}\!\left[
      \forall h \in \mathcal{H},\; 
      \varepsilon_{\mathrm{int}}(h) \ge - t_{2,h}
    \right]
    \ge 1 - \tfrac{1}{2}\delta.
\end{equation*}

\minihead{Combining the bounds.}
Applying the union bound over the two events, with probability at least 
$1 - \delta$, for all $h \in \mathcal{H}$ simultaneously:
\begin{equation*}
    \mathcal{R}(h) - \frac{1}{mn} \sum_{i=1}^{m} \sum_{j=1}^n r\!\bigl(x_i, g(x_i, \xi_j; h)\bigr)
    = \varepsilon_{\mathrm{ext}}(h) + \varepsilon_{\mathrm{int}}(h)
    \le t_{1,h} + t_{2,h}.
\end{equation*}
Let $B(h) \coloneqq \log(1/P(h)) + \log(2/\delta)$. Then
\begin{equation*}
    t_{1,h} + t_{2,h} 
    = \Delta\!\left(\sqrt{\frac{B(h)}{2m}} + \sqrt{\frac{B(h)}{2mn}}\right)
    = \Delta\!\left(1 + \sqrt{\frac{1}{n}}\right)\sqrt{\frac{B(h)}{2m}}.
\end{equation*}

\paragraph{Upper-bounding the Kolmogorov complexity.}
Since the true Kolmogorov complexity $K(h)$ is uncomputable, we upper-bound
$\log(1/P(h))$ using a practical compression size $C(h)$ in bits. Encoding $h$
with a self-delimiting prefix code \citep{lotfi2023non} gives
\begin{equation*}
    \log\frac{1}{P(h)} = \log Z + K(h)\log 2 
    \;\le\; C(h)\log 2 + 2\log C(h).
\end{equation*}
Substituting into $B(h)$ and rearranging yields the final bound and completes the proof.
\end{proof}

\subsection{Proof of Corollary \ref{cor:subsampled-eval}}
\label{subsec:subsampling-proof}


\begin{proof}
We decompose the gap between the population reward and the subsampled
empirical estimate by inserting the \emph{full-data} estimator
$\hat{R}_{m, n}(h) \coloneqq \tfrac{1}{mn}\sum_{i=1}^{m}\sum_{j=1}^{n}
r(x_i, g(x_i, \xi_{i,j}; h))$ as an intermediate quantity:
\begin{equation*}
    \mathcal{R}(h) - \hat{R}_{m_0, n}(h)
    \;=\; \underbrace{\bigl[\mathcal{R}(h) - \hat{R}_{m, n}(h)\bigr]}_{\text{full-sample gap}}
    \;+\; \underbrace{\bigl[\hat{R}_{m, n}(h) - \hat{R}_{m_0, n}(h)\bigr]}_{ 
      \varepsilon_{\mathrm{sub}}(h)\;\text{(subsampling gap)}}.
\end{equation*}
This decomposition allows us to invoke Theorem~\ref{theorem:bound}
directly as a black box for the first term, while the second term reduces
to a clean Hoeffding bound on subsampling from a fixed finite population.
We split the confidence budget to $2\delta/3$ for the first term and $\delta/3$
for the second term. 
 
\minihead{Bounding the full-sample gap.}
Applying Theorem~\ref{theorem:bound} with confidence parameter
$2\delta/3$, with probability at least $1 - 2\delta/3$, for all
$h \in \mathcal{H}$ simultaneously,
\begin{equation*}
    \mathcal{R}(h) - \hat{R}_{m, n}(h)
    \;\le\; \Delta\!\left(\sqrt{\tfrac{1}{n}} + 1\right)
    \sqrt{\tfrac{B_{3}(h)}{2m}},
\end{equation*}
where $B_{3}(h) \coloneqq C(h)\log 2 + 2\log C(h) + \log(3/\delta)$.
 
\minihead{Bounding the subsampling gap.}
Condition on the training prompts $\{x_i\}_{i=1}^{m}$ and the noise vectors
$\{\xi_{i,j}\}_{i \in [m],\, j \in [n]}$. Define the per-prompt empirical
averages
\begin{equation*}
    s_i(h) \;\coloneqq\; \tfrac{1}{n}\sum_{j=1}^{n}
    r\!\bigl(x_i, g(x_i, \xi_{i,j};\, h)\bigr) \;\in\; [a,\, a+\Delta],
    \qquad i \in [m].
\end{equation*}
Once we condition on the prompts and noise, the values $s_1(h), \dots, s_m(h)$
are \emph{fixed} bounded scalars. The full-data estimator is their mean,
$\hat{R}_{m, n}(h) = \tfrac{1}{m}\sum_{i=1}^{m} s_i(h)$, and the subsampled
estimator is the corresponding subsample mean,
$\hat{R}_{m_0, n}(h) = \tfrac{1}{m_0}\sum_{k=1}^{m_0} s_{i_k}(h)$.
 
Because the indices $\{i_k\}_{k=1}^{m_0}$ are drawn uniformly at random from
$[m]$ and are independent of $h$, the random variables
$s_{i_1}(h), \dots, s_{i_{m_0}}(h)$ are i.i.d.\ draws (with replacement) from
the empirical distribution over $\{s_1(h), \dots, s_m(h)\}$, with common mean
$\hat{R}_{m, n}(h)$ and values in $[a, a+\Delta]$. Applying Hoeffding's
inequality with budget $\delta/3$,
\begin{equation*}
    \mathbb{P}\!\left[
      \varepsilon_{\mathrm{sub}}(h) \,\le\, 
      \Delta\sqrt{\tfrac{\log(3/\delta)}{2 m_0}}
    \right] \,\ge\, 1 - \tfrac{1}{3}\delta.
\end{equation*}
Crucially, no union bound over $\mathcal{H}$ is required for this term:
$h$ is fixed before the subsample is drawn, so the bound holds for any
\emph{single} $h$ and the penalty is independent of $C(h)$. Marginalizing
over the conditioning preserves the unconditional probability statement.
 
\minihead{Combining the bounds.}
By a union bound over the two events, with probability at least $1 - \delta$,
for all $h \in \mathcal{H}$ simultaneously,
\begin{equation*}
    \mathcal{R}(h) - \hat{R}_{m_0, n}(h)
    \;\le\; \Delta\sqrt{\tfrac{\log(3/\delta)}{2 m_0}}
    \;+\; \Delta\!\left(\sqrt{\tfrac{1}{n}} + 1\right)
    \sqrt{\tfrac{B_3(h)}{2m}},
\end{equation*}
which is the stated bound. This completes the proof.
\end{proof}

\subsection{Out-of-Distribution Extension}
\label{subsec:app-ood-bound}

\minihead{Setup.}
Let $\mathcal{Q}_0$ denote an arbitrary target prompt distribution, possibly
different from $\mathcal{P}_0$. Let
$\{x'_k\}_{k=1}^{m'} \overset{\mathrm{iid}}{\sim} \mathcal{Q}_0$ denote an
OOD evaluation corpus of size $m'$, with independent Gumbel noise vectors
$\xi'_{k,j} \overset{\mathrm{iid}}{\sim} \mathcal{P}_\xi$ for
$k \in [m'],\, j \in [n']$. The \emph{target population reward} of a
hypothesis $h$ is
\begin{equation*}
    \mathcal{R}_{\mathcal{Q}}(h)
    \;=\; \mathbb{E}_{x \sim \mathcal{Q}_0}\!\left[
        \mathbb{E}_{\xi \sim \mathcal{P}_\xi}\!\left[
            r(x,\, g(x, \xi;\, h))
        \right]
    \right],
\end{equation*}
and the corresponding empirical OOD reward is
\begin{equation*}
    \hat{R}^{\mathcal{Q}}_{m', n'}(h)
    \;=\; \frac{1}{m'n'} \sum_{k=1}^{m'} \sum_{j=1}^{n'}
        r(x'_k,\, g(x'_k, \xi'_{k,j};\, h)).
\end{equation*}
We assume a \emph{labeled OOD oracle}: the RLVR verifier can evaluate
$r(x', y)$ for any OOD prompt $x'$ and generation $y$. This holds for
verifiable-reward tasks whose verifiers extend across domains (e.g.,
symbolic math verifiers on new problems, unit tests on new programs,
SQL execution on new databases).

\begin{theorem}[PAC-Bayes bound for RLVR under distribution shift]
\label{theorem:ood-bound}
Fix any $\delta \in (0, 1)$ and let $h^\star \in \mathcal{H}$ be any
hypothesis chosen independently of the OOD evaluation set
$\{(x'_k, \xi'_{k,1}, \dots, \xi'_{k, n'})\}_{k=1}^{m'}$. Then
\begin{equation*}
    \mathbb{P}\!\left[
        \mathcal{R}_{\mathcal{Q}}(h^\star) \,\ge\,
        \hat{R}^{\mathcal{Q}}_{m', n'}(h^\star)
        - \Delta \!\left(\sqrt{\tfrac{1}{n'}} + 1\right)
        \sqrt{\tfrac{\log(2/\delta)}{2 m'}}
    \right] \,\ge\, 1 - \delta,
\end{equation*}
where the probability is taken over the OOD sample.
\end{theorem}

\begin{proof}[Proof of Theorem~\ref{theorem:ood-bound}]
Following the decomposition of Remark~\ref{rem:error-decomposition}, we
split the OOD generalization gap into extrinsic and intrinsic components:
\begin{equation*}
\begin{aligned}
    &\mathcal{R}_{\mathcal{Q}}(h^\star) - \hat{R}^{\mathcal{Q}}_{m', n'}(h^\star) \\
    &\qquad = \underbrace{\mathbb{E}_{x \sim \mathcal{Q}_0}[\bar r(x; h^\star)]
        - \tfrac{1}{m'} \textstyle\sum_{k=1}^{m'} \bar r(x'_k; h^\star)
    }_{\varepsilon^{\mathcal{Q}}_{\mathrm{ext}}(h^\star)\;\text{(extrinsic)}}
    + \underbrace{\tfrac{1}{m'} \textstyle\sum_{k=1}^{m'} \bar r(x'_k; h^\star)
        - \hat{R}^{\mathcal{Q}}_{m', n'}(h^\star)
    }_{\varepsilon^{\mathcal{Q}}_{\mathrm{int}}(h^\star)\;\text{(intrinsic)}},
\end{aligned}
\end{equation*}
where $\bar r(x; h) := \mathbb{E}_{\xi \sim \mathcal{P}_\xi}[r(x, g(x, \xi; h))]$.
We bound each term separately.

\minihead{Bounding the extrinsic error.}
Since $r \in [a, a+\Delta]$, the conditional expectation
$\bar r(\cdot; h^\star)$ is likewise valued in $[a, a+\Delta]$. By hypothesis,
$h^\star$ is a deterministic function of the in-distribution training data
only, and is therefore independent of $\{x'_k\}_{k=1}^{m'}$. The sequence
$\bar r(x'_1; h^\star), \dots, \bar r(x'_{m'}; h^\star)$ consists of $m'$
i.i.d.\ bounded random variables with common mean
$\mathbb{E}_{x \sim \mathcal{Q}_0}[\bar r(x; h^\star)]$. By Hoeffding's
inequality, for any $t_1 > 0$,
\begin{equation*}
    \mathbb{P}\!\left[\varepsilon^{\mathcal{Q}}_{\mathrm{ext}}(h^\star) \ge -t_1\right]
    \;\ge\; 1 - \exp\!\left(-\tfrac{2 m' t_1^2}{\Delta^2}\right).
\end{equation*}
Allocating half of the confidence budget to this event and solving
$\exp(-2m't_1^2/\Delta^2) = \delta/2$ yields
\begin{equation*}
    t_1 \;=\; \Delta \sqrt{\tfrac{\log(2/\delta)}{2 m'}}.
\end{equation*}
\emph{No union bound over $\mathcal{H}$ is required}: the bound is stated
for the specific (pre-selected) $h^\star$. This is the key structural
difference from the proof of Theorem~\ref{theorem:bound} in
Appendix~\ref{subsec:app-bound-proof} and is what eliminates the
$C(h)\log 2 + 2\log C(h)$ term.

\minihead{Bounding the intrinsic error.}
Condition on the OOD prompts $x'_1, \dots, x'_{m'}$. The OOD noise vectors
$\{\xi'_{k,j}\}_{k \in [m'], j \in [n']}$ are drawn independently across all
$(k, j)$ and are independent of both $\{x'_k\}$ and $h^\star$. Therefore the
$m'n'$ random variables
$\{r(x'_k, g(x'_k, \xi'_{k,j}; h^\star))\}_{k,j}$ are mutually independent
conditional on $\{x'_k\}$, with each bounded in $[a, a+\Delta]$; reward
distributions may differ across prompts, but Hoeffding's inequality requires
only independence and uniform boundedness. For any $t_2 > 0$,
\begin{equation*}
    \mathbb{P}\!\left[\varepsilon^{\mathcal{Q}}_{\mathrm{int}}(h^\star) \ge -t_2
    \,\big|\, x'_1, \dots, x'_{m'}\right]
    \;\ge\; 1 - \exp\!\left(-\tfrac{2 m' n' t_2^2}{\Delta^2}\right).
\end{equation*}
Marginalizing over $\{x'_k\}$ preserves the unconditional probability
statement. Allocating the remaining $\delta/2$ and solving yields
\begin{equation*}
    t_2 \;=\; \Delta \sqrt{\tfrac{\log(2/\delta)}{2 m' n'}}.
\end{equation*}
Again, no union bound over $\mathcal{H}$ is required.

\minihead{Combining the bounds.}
By a union bound over the two events, with probability at least $1 - \delta$,
\begin{equation*}
    \mathcal{R}_{\mathcal{Q}}(h^\star) - \hat{R}^{\mathcal{Q}}_{m', n'}(h^\star)
    \;\ge\; -t_1 - t_2.
\end{equation*}
Factoring $\sqrt{\log(2/\delta)/(2m')}$ out of both terms gives
\begin{equation*}
    t_1 + t_2
    \;=\; \Delta\!\left(1 + \sqrt{\tfrac{1}{n'}}\right)
    \sqrt{\tfrac{\log(2/\delta)}{2 m'}}
    \;=\; \Delta\!\left(\sqrt{\tfrac{1}{n'}} + 1\right)
    \sqrt{\tfrac{\log(2/\delta)}{2 m'}},
\end{equation*}
which yields the stated inequality and completes the proof.
\end{proof}
\section{Experimental Setup and Hyperparameters}
\label{appendix:exp_setup}

We detail the settings of the Progressive RLVR pipeline 
used to produce every checkpoint and number reported in the main paper. All 
experiments use Qwen3.5-4B as the base model. The full training and evaluation 
pipelines are submitted as part of the supplementary material.

\subsection{Teacher Training}
\label{appendix:teacher_training}
Each of the four per-domain teachers is a rank-$64$ LoRA adapter on
top of Qwen3.5-4B, trained via the 
Tinker\footnote{\url{https://github.com/thinking-machines-lab/tinker-cookbook}} 
reinforcement-learning cookbook against verifier-defined rewards on the same 
dataset that will later be used for student rollouts. We keep the recipe 
structure identical across domains. We use GRPO-style rollouts 
\cite{shao2024deepseekmath} at temperature $1.0$, group size $8$, $64$ groups 
per training batch, and a maximum generation length of $4096$ tokens. We conduct
a hyperparameter sweep over learning rates $\{10^{-5}, 5\times10^{-5}, 10^{-4}\}$,
and select the checkpoint with the highest training reward. We train for 200
optimizer steps for all runs, which is enough to achieve convergence.

We apply multi-turn RL on the Text-to-SQL domain. In each turn, the model issues
\texttt{<tool\_call>}-wrapped SQL probes, observes execution results, and 
ultimately commits an answer inside a \texttt{<solution>} block, with up to $5$ 
turns per trajectory. The reward is $+1$ if the committed query's result-set 
matches ground truth, $0$ if it executes but is wrong, and $-1$ for format 
violations (no \texttt{<solution>} tag, leaked nested markup, etc.). Math and
general domains use \texttt{math\_verify} and \texttt{sympy} parsing of
\verb|\boxed{}|-style answers (binary $\{0,1\}$, with a hard $30$~s timeout per
response). Code domain executes the generated code against unit tests via pytest
in a sandboxed subprocess (binary $\{0,1\}$, $30$~s timeout). The full
configuration is summarized in Table~\ref{tab:training_hparams}.

\subsection{On-Policy Distillation Training}
\label{appendix:opd_training}

All runs share the hyperparameters in Table~\ref{tab:training_hparams}
and were trained for 200 update steps each, which is sufficient to achieve
convergence.

\begin{table}[t]
  \centering
  \footnotesize
  \caption{Training configuration for the teacher RLVR and on-policy
  distillation stages. Dashes (---) indicate the setting is not applicable or
  follows the framework default.}
  \label{tab:training_hparams}
  \begin{tabular}{lll}
    \toprule
    \textbf{Setting} & \textbf{Teacher Training} & \textbf{On-Policy Distillation} \\
    \midrule
    \multicolumn{3}{l}{\emph{Adapter}} \\
    Type & LoRA & TinyLoRA\\
    Rank & $64$ & $16$ \\
    Trainable coefficients $|v|$ & --- & $16$ \\
    Target modules & --- & all linear \\
    \midrule
    \multicolumn{3}{l}{\emph{Optimization}} \\
    Optimizer & Adam ($\beta_1\!=\!0.9$, $\beta_2\!=\!0.95$, $\epsilon\!=\!10^{-8}$) & AdamW \\
    Final learning rate & $1\!\times\!10^{-5}$ (math/code/general); $1\!\times\!10^{-4}$ (sql) & $1\!\times\!10^{-4}$ \\
    Warmup steps & $0$ & $0$ \\
    Weight decay & $0$ & $0$ \\
    Training steps & $200$ & $200$ \\
    Batch size (groups per step) & $64$ & $64$ \\
    \midrule
    \multicolumn{3}{l}{\emph{Loss}} \\
    Reward & verifier-defined (per-domain) & per-token KL to teacher \\
    Policy loss & importance sampling & importance sampling \\
    \midrule
    \multicolumn{3}{l}{\emph{Rollouts}} \\
    Multi-turn & sql only ($\texttt{max\_turns}\!=\!5$) & inherits domain setting \\
    Group size (rollouts per prompt) & $8$ & $8$ \\
    Sampling temperature & $1.0$ & $1.0$ \\
    \texttt{top\_p}, \texttt{top\_k} & $1.0$, $-1$ & $1.0$, $-1$ \\
    Max generation length & $4096$ tokens & $4096$ tokens \\
    Inference engine & Tinker default & vLLM \\
    Seed & $0$ & $0$ \\
    \bottomrule
  \end{tabular}
\end{table}

\subsection{Evaluation Protocol}
\label{appendix:eval}
For each cell we evaluate on $4096$ problems randomly drawn from the
corresponding eval-task dataset, generating 64 independent completions per 
problem. We use the same sampling configuration as training rollouts. Inference 
is performed with vLLM. The reward function matches the one used in the 
corresponding training domain. 

\section{Encoding Scheme and Description Length Accounting}
\label{sec:app-encoding}

The compression-based PAC-Bayes penalty in
Theorem~\ref{theorem:bound} depends on $C(h)$, the description length
in bits of the hypothesis $h$ under a fixed encoding scheme. The bound
is valid only if $C(h)$ is sufficient to fully reconstruct $h$ given a
\emph{hypothesis-independent} decoder. In this appendix, we make the
encoding scheme used in our experiments explicit. We enumerate every
ingredient of the student model, declare which ingredients are fixed
across the hypothesis space $\mathcal{H}$ (and therefore cost zero
per-$h$ bits), describe the encoding of the per-$h$ information, and
tabulate the resulting $C(h)$ used in our reported bounds.

\subsection{Decoder Specification}
\label{subsec:app-encoding-fixed}

The following ingredients are fixed before training and are identical
for every $h \in \mathcal{H}$. They are part of the decoder
specification and contribute zero bits to $C(h)$.

\begin{itemize}[leftmargin=*]
\item \textbf{TinyLoRA adapter placement.} Adapters are inserted at a
fixed set of $L = 249$ modules of \texttt{Qwen3.5-4B}: the
language-model head, and for each transformer block, the attention
projection matrices and the MLP up/gate/down projections.
\item \textbf{TinyLoRA truncated SVD basis $(U, \Sigma, V)$.} The bases
are a deterministic function of the frozen pre-trained weights $W_0$
\cite{morris2026learning} and need not be transmitted.
\item \textbf{TinyLoRA random projection matrices $\{P_i\}_{i=1}^{u}$.}
Each $P_i \in \mathbb{R}^{r \times r}$ is generated from a single
master seed (42) shared across all $h \in \mathcal{H}$. The seed is
part of the decoder spec.
\item \textbf{LoRA scaling $\alpha/r$, sampling temperature $\tau$,
maximum sequence length $T$.} Single values fixed for the entire
hypothesis space.
\item \textbf{Number of trainable components $u$ and rank $r$.} Both
$u$ and $r$ are fixed in advance.
\end{itemize}

Because all of the above are hypothesis-independent, the only per-$h$
quantities that the decoder must receive are the trainable vector
$\{v^{(\ell)}\}_{\ell=1}^{L}$ with $v^{(\ell)} \in \mathbb{R}^{u}$,
giving a total of $N = u \cdot L = 3984$ float components per $h$, and 
quantization levels, together with the small header described next.

\subsection{Description Length Breakdown}
\label{subsec:app-encoding-perh}

We quantize each component of $v$ to one of $k$ levels using a
single \emph{global} codebook $\mathcal{C}(h) = \{c_1, \dots, c_k\}
\subset \mathbb{R}$ that is shared across all $L$ adapter sites for a
given hypothesis $h$. The codebook entries are fit per-$h$ during
training and therefore must be transmitted as part of $C(h)$.

\minihead{Index stream.} Each component $v_i^{(\ell)}$ is encoded
as the index $j \in \{1, \dots, k\}$ of its codebook entry. Let
$n_j(h)$ denote the empirical count of level $j$ across the $N$
components. We use the universal Krichevsky--Trofimov (KT) arithmetic
code \cite{krichevsky1981performance} on the index stream, whose
self-decodable length is
\begin{equation*}
    L_{\mathrm{KT}}(h) \;=\; \left\lceil
    -\log_2 \frac{\prod_{j=1}^{k} \Gamma\!\left(n_j(h) + \tfrac{1}{2}\right)\big/\Gamma\!\left(\tfrac{1}{2}\right)^{k}}
    {\Gamma\!\left(N + \tfrac{k}{2}\right)\big/\Gamma\!\left(\tfrac{k}{2}\right)}
    \right\rceil.
\end{equation*}
The KT code matches the empirical entropy
$H(h) = -\sum_j (n_j/N) \log_2 (n_j/N)$ of the index stream up to a
$\tfrac{k-1}{2} \log_2 N$ overhead and requires no separate frequency
header. As a sanity check, we also report (i) the empirical
arithmetic code at length $\lceil N \cdot H(h) \rceil + 2$ bits
augmented by an explicit $(k-1)\lceil \log_2(N+1) \rceil$-bit
frequency header and (ii) a uniform fixed-length code at
$N \lceil \log_2 k \rceil = 3N$ bits with no header.

\minihead{Codebook header.} Each of the $k$ codebook entries $c_j$
is transmitted as an fp16 scalar, contributing $16k$ bits per $h$.

\minihead{Quantization-level selection.} The quantization level $k$
defines the encoding scheme and therefore the hypothesis space
$\mathcal{H}_k$ itself; different $k$ values yield different
$\mathcal{H}_k$. We searched
$k \in \{2, 3, 4, 5, 6, 7, 8, 16, 32, 64\}$ (a grid of $10$ values)
and report the tightest bound across this grid. To make the resulting
bound valid simultaneously across all $10$ encoding schemes, we form
a meta-prior by spending an additional
$\lceil \log_2 10 \rceil = 4$ bits to identify the chosen $k$. This
contribution is independent of $h$ but is per-$h$ in the sense that it
must be added to $C(h)$ for every reported bound.

The total per-$h$ description length is therefore
\begin{equation}
\label{eq:app-encoding-Ch}
    C(h_k) \;=\; L_{\mathrm{KT}}(h_k)
    \;+\; \underbrace{16k}_{\text{codebook}}
    \;+\; \underbrace{4}_{\text{$k$-selection}}.
\end{equation}
The self-delimiting prefix-code overhead $2 \log C(h)$ that converts
$C(h) \log 2$ to $\log(1/P(h))$ in Theorem~\ref{theorem:bound} is
already embedded in the bound's penalty and is not included in
\eqref{eq:app-encoding-Ch}.

\subsection{Description length per domain}
\label{subsec:app-encoding-numbers}

Table~\ref{tab:app-encoding-numbers} reports $C(h)$ for the four
trained student adapters. The codebook entries differ across domains,
reflecting different empirical scales of the trainable vector $v$, but
the codebook \emph{size} is identical ($k = 5$, $80$ bits) and the
$k$-selection cost is identical ($4$ bits). The empirical entropy
$H(h)$ varies by domain, producing a $\sim 2.4\times$ spread in the
index-stream length: mathematics and Text-to-SQL exhibit broader
codebook usage ($H(h) \approx 1.9$ bits/parameter), whereas code and
general-knowledge reasoning concentrate over $80\%$ of the components
on a single near-zero level ($H(h) \approx 0.9$ bits/parameter).

\begin{table}[t]
\centering
\caption{Description length $C(h)$ of the student
adapter at $k = 5$ codebook levels, broken down by encoding component.
``Index (KT)'' is the universal Krichevsky--Trofimov arithmetic code;
``Index (uniform)'' is the worst-case $3$-bit fixed-length code. The
per-$h$ header consists of the $k$ fp16 codebook entries ($80$ bits)
plus the $4$-bit $k$-selection cost ($\lceil \log_2 10 \rceil$).}
\label{tab:app-encoding-numbers}
\footnotesize
\begin{tabular}{@{}lcccccc@{}}
\toprule
& $H(h)$ & Index (KT) & Index (uniform) & Header & Total $C(h)$ (KT) & Total $C(h)$ (uniform) \\
Domain & (bits/param) & (bits) & (bits) & (bits) & (bits) & (bits) \\
\midrule
Mathematics       & $1.912$ & $7642$ & $11952$ & $84$ & $7726$ & $12036$ \\
Code              & $0.833$ & $3342$ & $11952$ & $84$ & $3426$ & $12036$ \\
General reasoning & $0.905$ & $3626$ & $11952$ & $84$ & $3710$ & $12036$ \\
Text-to-SQL       & $1.978$ & $7903$ & $11952$ & $84$ & $7987$ & $12036$ \\
\bottomrule
\end{tabular}
\end{table}

In every domain, the per-$h$ header (codebook plus $k$-selection)
accounts for at most $2.5\%$ of $C(h)$ under the KT code; the dominant
contribution is the index stream itself. The uniform code produces a
domain-independent $C(h) = 12{,}036$ bits, conservative by
$\approx 1.5\times$ to $3.5\times$ relative to KT depending on domain.

\section{Out-of-domain Evaluation}
\label{sec:app-ood-evals}
\begin{wrapfigure}{r}{0.5\textwidth}
    \centering
    \includegraphics[width=0.4\textwidth]{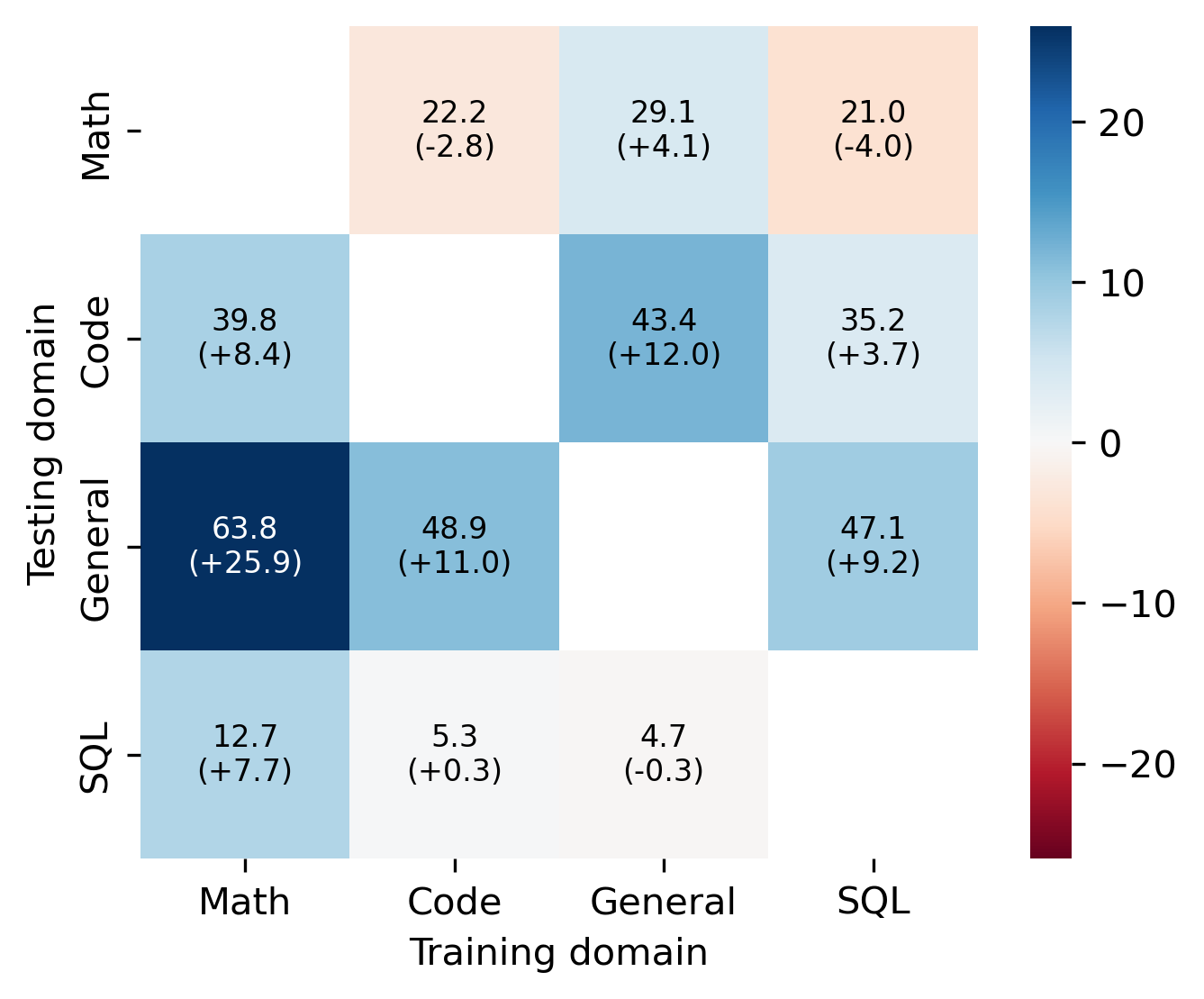}
    \caption{Cross-domain generalization bounds. 
    }
    \label{fig:ood}
\end{wrapfigure}
We now extend the experiments to evaluating the OOD performance of the models
trained via Progressive RLVR. For each pair of domains 
($\mathcal{P}_0$, $\mathcal{Q}_0$), we evaluate the quantized student trained on 
$\mathcal{P}_0$ against disjoint corpus from $\mathcal{Q}_0$. In 
Figure~\ref{fig:ood}, we show cross-domain bounds for each ordered pair of 
domains: a student trained on the source (column) is evaluated on a labeled 
corpus from the target (row), and the cell shows the resulting bound (with 
improvement over the base model in parentheses). All twelve off-diagonal pairs 
yield non-vacuous bounds, and nine show improvements over the base model, 
indicating that RLVR specialists could preserve useful capability outside 
their training domain. 

We find that the transfer pattern is markedly asymmetric. Math is the strongest 
source domain, with positive transfer to every other target and the largest 
single improvement of any pair (+25.9 points on General). In contrast, Math is 
also the hardest target. Other sources transfer neutrally or negatively, with 
the worst case being SQL$\to$Math at $-4.0$ points. This asymmetry suggests that 
RLVR on mathematical reasoning yields broadly transferable representations, 
while reasoning patterns from other domains do not recover the symbolic 
precision mathematics requires. General reasoning is the easiest target, resulting 
in gain from every source, while SQL is the most isolated, with all cross-domain 
bounds below 13\% accuracy.
\section{Generalization Bound for the Rewards on Text-to-SQL Tasks}
\label{sec:app-sql-reward}

\begin{wrapfigure}{r}{0.5\textwidth}
    \centering
    \includegraphics[width=0.5\textwidth]{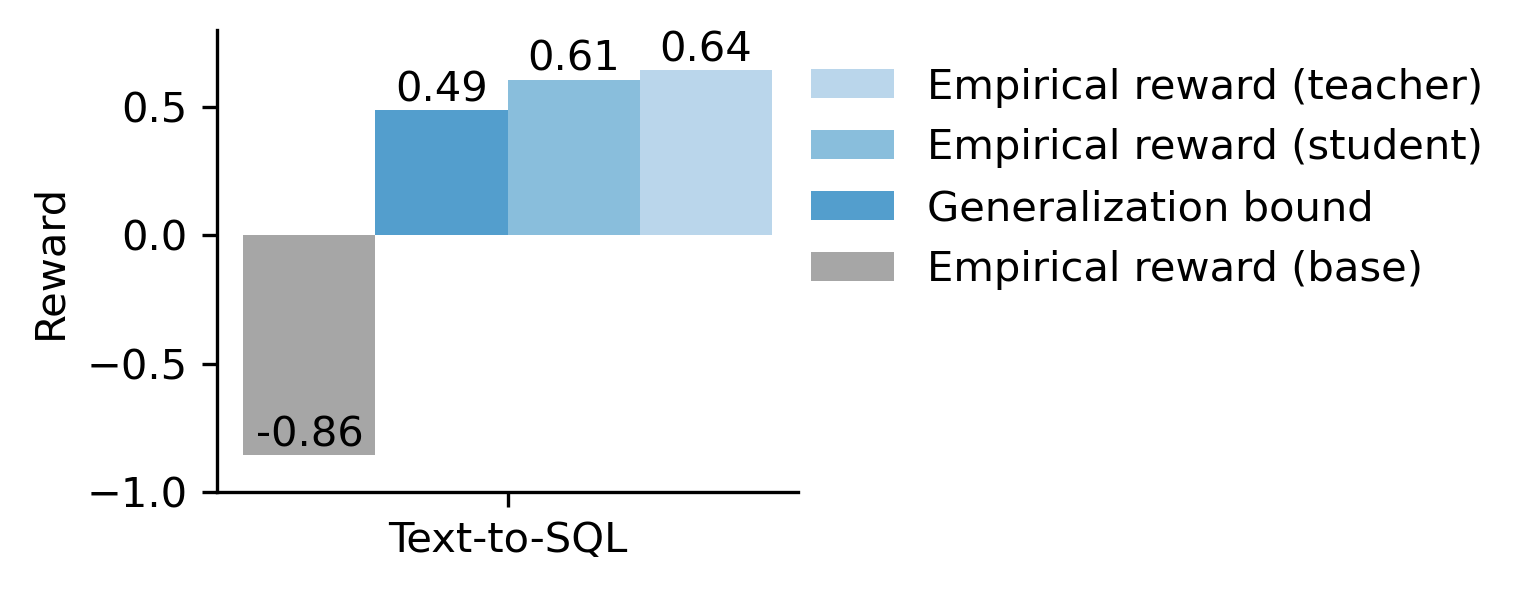}
    \caption{Non-vacuous accuracy bound for Text-to-SQL at $\delta = 0.05$.}
    \label{fig:sql-accuracy}
\end{wrapfigure} 

The Text-to-SQL bound reported in Figure~\ref{fig:overall-bounds} is computed on 
task accuracy rather than the tri-level reward ($r \in \{-1, 0, 1\}$) 
used in training. Here, we re-calculate the bound under the raw reward. 
Theorem~\ref{theorem:bound} applies directly with $\Delta = 2$, yielding a bound 
on the expected reward.

We show the result in Figure~\ref{fig:sql-accuracy}. We obtain a non-vacuous 
accuracy bound of $0.56$ for the trained model, exceeding the base model's 
empirical accuracy of $0.04$ by $0.52$ and lying within $0.08$ of the student's 
empirical accuracy of $0.64$. This tightness is consistent with the 
binary-reward bounds for the other three domains (Figure~\ref{fig:overall-bounds}).



\end{document}